\documentclass{article}

\usepackage{PRIMEarxiv}

\usepackage[utf8]{inputenc} 
\usepackage[T1]{fontenc}    
\usepackage{hyperref}       
\usepackage{url}            
\usepackage{booktabs}       
\usepackage{amsfonts}       
\usepackage{nicefrac}       
\usepackage{microtype}      
\usepackage{lipsum}
\usepackage{fancyhdr}       
\usepackage{graphicx}       
\graphicspath{{media/}}     
\usepackage{amssymb}
\usepackage{latexsym}

\usepackage{graphicx}
\usepackage{mathtools}
\usepackage{amsfonts}
\usepackage{pifont}
\usepackage{booktabs}
\usepackage{todonotes}
\usepackage{siunitx}
\usepackage{microtype}
\usepackage{dirtytalk}
\usepackage{amsthm}
\usepackage{makecell} 
\usepackage{xcolor}
\usepackage{natbib}
\usepackage{algorithm}%
\usepackage{algorithmicx}%
\usepackage{algpseudocode}%

\DeclarePairedDelimiterX{\inner}[2]{\langle}{\rangle}{#1, #2}

\newcommand{\Z}{\mathbb{Z}}
\newcommand{\Q}{\mathbb{Q}}
\newcommand{\N}{\mathbb{N}}
\newcommand{\C}{\mathbb{C}}
\newcommand{\R}{\mathbb{R}}

\DeclareMathOperator{\sech}{sech}

\newcommand{\algorithmInput}[1]{\hspace*{\algorithmicindent} \textbf{Input} #1 \\}
\newcommand{\algorithmOutput}[1]{\hspace*{\algorithmicindent} \textbf{Output} #1}

\newtheorem{definition}{Definition}
\newtheorem{proposition}{Proposition}
\newtheorem{theorem}{Theorem}

\newtheorem{lemma}{Lemma}

\newtheorem{corollary}{Corollary}
\newtheorem{example}{Example}

\newcounter{ArasCounter}

\newcounter{GittaCounter}

\newcounter{HolgerCounter}

\newcounter{PhilippCounter}

\newcommand*\widebar[1]{\@ifnextchar^{{\wide@bar{#1}{0}}}{\wide@bar{#1}{1}}}

\title{Symbolic Recovery of Differential Equations: \\ The Identifiability Problem
}

\author{
  Philipp Scholl \\
  Ludwig-Maximilians-Universit\"at M\"unchen \\
  Munich \\
  Germany\\
  \texttt{scholl@math.lmu.de} \\
   \And
  Aras Bacho \\
  Ludwig-Maximilians-Universit\"at M\"unchen \\
  Munich \\
  Germany\\
  \texttt{bacho@math.lmu.de} \\
     \And
  Holger Boche \\
  Technische Universit\"at M\"unchen \\
  Munich Center for Quantum Science and Technology (MCQST) \\
  Munich Quantum Valley (MQV) \\
  Munich \\
  Germany\\
  \texttt{boche@tum.de} \\
     \And
  Gitta Kutyniok \\
  Ludwig-Maximilians-Universit\"at M\"unchen \\
  University of Troms\o{} \\
  DLR-German Aerospace Center \\
  Munich Center for Machine Learning (MCML) \\
  Munich \\
  Germany\\
  \texttt{kutyniok@math.lmu.de} \\
}

\begin{document}
\maketitle

\begin{abstract}
Symbolic recovery of differential equations is the ambitious attempt at automating the derivation of governing equations with the use of machine learning techniques. In contrast to classical methods which assume the structure of the equation to be known and focus on the estimation of specific parameters, these algorithms aim to learn the structure and the parameters simultaneously. While the uniqueness and, therefore, the identifiability of parameters of governing equations are a well-addressed problem in the field of parameter estimation, it has not been investigated for symbolic recovery. However, this problem should be even more present in this field since the algorithms aim to cover larger spaces of governing equations. In this paper, we investigate under which conditions a solution of a differential equation does not uniquely determine the equation itself. For various classes of differential equations, we provide both necessary and sufficient conditions for a function to uniquely determine the corresponding differential equation. We then use our results to devise numerical algorithms aiming to determine whether a function solves a differential equation uniquely. Finally, we provide extensive numerical experiments showing that our algorithms can indeed guarantee the uniqueness of the learned governing differential equation, without assuming any knowledge about the analytic form of function, thereby ensuring the reliability of the learned equation.
\end{abstract}


\section{Introduction} \label{sec:introduction}
For most of human history, scientists had to derive physical laws by hand. As the availability of data in physics is growing and artificial intelligence is becoming more powerful, engineers and physicists are aiming to apply machine learning to infer not only specific parameters of the governing laws from experimental data \citep{alessandrini1986identification, acar1993identification, knowles2001parameter} but additionally the structure of the governing laws \citep{bongard2007lipson, brunton2016}. If this is done in a symbolic, i.e., human-readable form, this is often referred to as symbolic regression or symbolic recovery of governing equations. This has the potential to advance our modeling and understanding of complex dynamics across disciplines. The pioneering paper \citep{schmidt2009lipson}, e.g., accurately computes dynamics used in mechanics and biology in experiments. \citet{xu2020} display the ability to compute equations used in environmental science, fluid dynamics, high-energy physics, and electronic science. Moreover, \citet{martius2017lampert} were able to learn dynamics from atomic physics and robotics.

The common assumption in this field is that the underlying physical law can be described by a function or a differential equation. In this paper, we concentrate on the case that the physical law is described by a differential equation. Therefore, we focus on methods, which take data from a solution $(u(t_i,x_j))_{ij}\subset\R$ of some ordinary differential equation (ODE)
\begin{equation} \label{eq:ode}
   F\left(u,\frac{\partial u}{\partial t}, \frac{\partial^2 u}{\partial^2 t}, ...\right)=\frac{\partial^n u}{\partial^n t}
\end{equation}
or partial differential equation (PDE)
\begin{equation} \label{eq:pde}
   F\left(u,\frac{\partial u}{\partial x_1},...,\frac{\partial u}{\partial x_m}, \frac{\partial^2 u}{\partial^2 x_1}, \frac{\partial^2 u}{\partial x_1 \partial x_2}, ...\right)=\frac{\partial^n u}{\partial^n t}
\end{equation}
as input and output the function $F$ describing the corresponding differential equation. If the structure of $F$ is known and the goal is to estimate specific parameters of it, we refer to it as the problem of \emph{parameter estimation}. If the structure and the parameters of $F$ are both unknown, we refer to it as \emph{symbolic recovery of governing equations}.

One important problem raised and discussed in the context of parameter estimation
is the uniqueness and, thus, identifiability of the parameters of interest \citep{pohjanpalo1978system, cobelli1980parameter, walter1982global, kohn1985determining, calderon2006inverse, alberti2019calderon}. The uniqueness of parameters is important for parameter identification as otherwise it hinders the algorithm from finding the correct parameters. Even though this topic has been well addressed in the context of parameter estimation, it has been so far ignored for the symbolic recovery of governing equations. Furthermore, we argue that the problem should be even more prominent in this field, since the approaches in this field decrease the assumptions on the governing equations and, therefore, potentially increase the ambiguities. For this reason, we aim to address the uniqueness problem of governing equations in the context of symbolic recovery of differential equations.

To establish that uniqueness is indeed not given in general but is highly desirable, we provide the following example, see \citet{Rudy2017DatadrivenDO}.

\begin{example} \label{ex:kdv-equation}
    The Korteweg–De Vries (KdV) equation is defined as
    \begin{equation}
        u_t=6uu_x-u_{xxx}
    \end{equation}
    and is solved by the 1-soliton (= self-reinforcing wave) solution
    \begin{equation} \label{eq:kdv-solution}
        u(t,x)=\frac{c}{2}\sech^2\left(\frac{\sqrt{c}}{2}(x-ct-a)\right)
    \end{equation}
    for $t\in\R_{>0}$ and $x\in\R$.
    Here, $\sech(x)=\frac{1}{\cosh(x)}$ is the hyperbolic secant. The problem for symbolic regression algorithms is that $u$ also solves the one-way wave equation $u_t=-cu_x$. Thus, the PDE is not identifiable, i.e., the learning algorithm cannot know which PDE to select and might output the wrong one \citep{Rudy2017DatadrivenDO}.
\end{example}

To the best of our knowledge, Example~\ref{ex:kdv-equation} in \citet{Rudy2017DatadrivenDO} is the only time this uniqueness problem is mentioned in the symbolic recovery of governing equation literature and it has never been tackled theoretically. Thus, this phenomenon poses a potential issue for scientific insights inferred using the existing methods, since it remains unclear if or when one can trust the learned equations. 

\subsection{Related Work} \label{sec:related-work}

\paragraph{Symbolic regression.}
Symbolic regression aims to identify the symbolic relation between input variables $x_1,...,x_n$ and the output variable $y$ without assuming a pre-defined model. Most approaches rely on Genetic Programming to heuristically search the space of functions \citep{augusto2000symbolic, schmidt2009distilling, Schmidt2010AgefitnessPO, cranmer2023interpretable}. Similarly, reinforcement learning has been applied to optimize for the simplest and most accurate formula \citep{petersen2019deep, sun2022symbolic} and combined with genetic programming \citep{mundhenk2021symbolic}. \citet{udrescu2020tegmark} simplify the problem recursively by applying common symmetries of physical laws. Other approaches translate the discrete optimization problem to a continuous one by parametrizing the space of functions as linear combination of functions \citep{McConaghy2011, brunton2016}, complex compositions of standard base functions \citep{martius2016extrapolation, sahoo2018learning,scholl2023parfam}, or Meijer-G functions \citep{alaa2019meijerg, crabbe2020learning}. Recently, transformers have been used to leverage pre-training on synthetic data to predict symbolic formulas directly \citep{biggio2021neural, kamienny2022end, holt2022deep}.

\paragraph{Symbolic recovery of differential equations.}
The algorithms mentioned above for symbolic regression are also suitable candidates to recover a differential equation symbolically, by using derivatives as the variables $x_i$ and/or $y$ \citep{la2021contemporary}. Sparse Identification of Non-linear Dynamics (SINDy) \citep{brunton2016}, which applies sparse regression to recover the closed-form ODE, has been studied extensively in this context and extended in various ways: \citet{Rudy2017DatadrivenDO} extended it to PDEs, \citet{kaheman2020sindy} to implicit PDEs, \citet{champion2019} introduced an autoencoder to learn the features and \citet{quade2018sparse} enabled it to handle abrupt changes in the system dynamics.
Other approaches utilize symbolic regression algorithms but reinforce them for the special case of differential equations, which is mostly motivated by the instability of the numerical derivative. One possible solution is to approximate the solution first with a neural network \citep{Hasan2020, both2021choudhury, stephany2022pde,chen2021physics}. \citet{qian2022dcode} instead reformulate the ODE as a variational problem to avoid the computation of the derivatives at all and \citet{kacprzyk2023dcipher} extend it to a general class of PDEs.

\paragraph{Parameter estimation and identifiability in differential equations.} 
A related and more mature field is parameter estimation in differential equations, which assumes that the structure of a differential equation is already known and only specific parameters have to be inferred from the data \citep{alessandrini1986identification, acar1993identification, knowles2001parameter}. This field is equipped not only with efficient algorithms but also with developed theories. The two main theoretical concerns are stability and uniqueness. \citet{acar1993identification}, for example, shows that the space-distributed diffusion coefficient in the steady-state diffusion equation can be stably recovered using least squares with Tikhonov regularisation. The uniqueness problem refers to the question if it is possible to recover all parameters of a differential equation using only data about the solution \citep{pohjanpalo1978system, cobelli1980parameter, walter1982global}. \citet{VAJDA1984} analyze the identifiability of linear, bilinear, polynomial, and rational ODEs using structural invariants aiming to unify many of the early results in this field. \citet{hong2020global} introduce an algebraic criterion along with an algorithm for identifiability of ODEs. \citet{ovchinnikov2021computing} take another approach for ODEs and introduce an algorithm that computes all identifiable functions and further introduces a result concerning the identifiability from multiple experiments.

\subsection{Our Contributions} \label{sec:our-contributions}

In this paper, we aim to address the uniqueness problem of symbolic recovery for differential equations.
We provide the first necessary and sufficient conditions for the uniqueness of PDEs for different function spaces for $F:\R^k\rightarrow\R$ in \eqref{eq:pde} in the context of symbolic recovery of differential equations. Our choices are linear, polynomial, algebraic, analytic, smooth, and continuous functions, as these cover large parts of important PDEs and their structure enables a systematic analysis of uniqueness. After proving the necessary and sufficient conditions for the uniqueness of PDEs, we apply these to specific classes of ODEs. We then introduce numerical algorithms to check the uniqueness of PDEs and show their correctness in numerical experiments.

Before we summarize our uniqueness results for the different classes of PDEs, we start by rigorously defining uniqueness. Note that we write $F(g)=f$, with functions $g:U\rightarrow\R^k$, $f:U\rightarrow \R$, $F:\R^k\rightarrow\R$, and $U\subset\R^{m+1}$ open, if $F(g(t,x))=f(t,x)$ for all $(t,x)\in U$. If $F(g(t,x))=0$ for all $(t,x)\in U$, we write $F(g)=0$.

\begin{definition} [Uniqueness] \label{def:unique}
    Let $u:U\rightarrow\R$ be a differentiable function on the open set $U\subset\R^{m+1}$. Let each $g_1,...,g_k:U\rightarrow\R$ be either a projection on one of the coordinates, any derivative of $u$ that exists, or the function $u$. Denote $g=(g_1,...,g_k):U\rightarrow\R^k$. Let $n\in\N\setminus\{0\}$ such that the $n^{th}$ time derivative of $u$ exists. Further, let $V$ be a set of functions which map from $\R^k$ to $\R$ and $F\in V$ such that
    \begin{equation} \label{eq:def-unique}
       \frac{\partial^n u}{\partial^n t}
       =F(g).
    \end{equation}
    We say that the function $u$ \emph{solves a unique differential equation described by functions in $V$ for $g=(g_1,...,g_k)$ with time derivative of $n^{th}$-order} if $F$ is the unique function in $V$ such that \eqref{eq:def-unique} holds. If $n=1$, we often omit the last part.
\end{definition}

With this definition we can express that $u$ as in Example~\ref{ex:kdv-equation} does not solve a unique differential equation described by polynomials for $g=(u,u_x,u_{xx},u_{xxx})$ with time derivative of $1^{st}$-order. Definition~\ref{def:unique} captures a broader class of PDEs than those described in \eqref{eq:ode} and \eqref{eq:pde}. For example, by setting $g:U\rightarrow\R^{m+l+1},\;g(t,x)=(t,x_1,...,x_m,u_{\alpha^1}(t,x),...,u_{\alpha^l}(t,x))$, Definition~\ref{def:unique} covers PDEs which directly depend on $t$ and $x$:
\begin{equation}
    \frac{\partial^n u}{\partial^n t}
       =F(t,x_1,...,x_m,u_{\alpha^1}(t,x),...,u_{\alpha^l}(t,x)),
\end{equation}

where $\alpha^1,...,\alpha^l\in\N^m$ are \emph{multi-indices}. For a multi-index $\alpha\in\N^m$, we write $u_\alpha=\frac{\partial^{|\alpha|} u}{\partial^{\alpha_1}x_1...\partial^{\alpha_m}x_m}$ with $|\alpha|=\sum_{i=1}^m\alpha_i$. Furthermore, Definition~\ref{def:unique} covers ODEs, e.g., an autonomous second-order ODE
\begin{equation}
    u_{tt}=F(u,u_t)
\end{equation}
can be expressed by setting $n=2$ and $g=(g_1,g_2)=(u,u_t)$. We will use this flexibility for various classes of ODEs in Section~\ref{sec:uniqueness-of-odes}. 

Throughout this paper, let $u$, $U$, $g_1,...,g_k$ and $n$ be as described in Definition~\ref{def:unique}. 
Our results for the different function classes are summarized in Table~\ref{tab:uniqueness-simple-pde}. This shows that a function solves a unique linear PDE if and only if the functions $g_1,...,g_k$ are linearly independent. For polynomial or algebraic PDEs\textemdash{}meaning that the function $F:\R^k\rightarrow\R$ is a polynomial or algebraic function\textemdash{}we show that the identifiability is equivalent to various properties from algebraic geometry. This yields the strong statement, that for algebraic solutions $u$ the uniqueness is equivalent to the Jacobian of $g$ having full rank for some $t,x\in U$. Interestingly, this is the sufficient condition for continuously differentiable functions $u$ for the uniqueness of analytic PDEs. In Section~\ref{sec:uniqueness-analytic-PDEs} we will also prove that it is not a necessary one. Continuous and differentiable PDEs, i.e., the function $F:\R^k\rightarrow\R$ is continuous or differentiable, are the most complicated PDEs as uniqueness for them is equivalent to the density of the image of $g$.

\begin{table}[h]
    \centering
    \caption{Uniqueness for $F(g_1,...,g_k)=\frac{\partial^n u}{\partial^n t}$. The second column shows conditions that are equivalent or sufficient to the uniqueness of a PDE in the class of the first column.}
    \label{tab:uniqueness-simple-pde}
    \begin{tabular}{ |c|c| } 
    \hline
     PDE class & Equivalent (E) or sufficient (S) for uniqueness  \\ 
     \hline \hline
     Linear PDE & (E) $g_1,...,g_k$ are  linearly independent \\ 
     \hline
     \makecell{Polynomial or \\ algebraic PDE} & \makecell{(E) There is no hypersurface/algebraic set $A\neq\R^k$ such that $\mathcal{D}\subset A$ \\
     (E) $\mathcal{D}$ has a non trivial "correspondence $I$" \\
     (E) Assuming $u$ is algebraic: \\ $k\leq m+1 $ and Jacobian $J_g(t,x)$ has full rank for some $t,x\in U$}  \\ 
     \hline
     Analytic PDE &  \makecell{(E) There is no C-analytic set $A\neq\R^k$ such that $\mathcal{D}\subset A$ \\ 
     (S) There is no analytic set $A\neq\R^k$ such that $\mathcal{D}\subset A$ \\
     (S) $\lambda^k(\mathcal{D})>0$ \\ 
     (S) The Jacobian $J_g(t,x)$ is full-rank for some $t,x\in U$}  \\ 
    \hline
     $C^p$, $0\leq p \leq \infty$  & (E) $\overline{\mathcal{D}}=\R^k$ \\ 

     \hline
    \end{tabular}

\end{table}

\subsection{Outline}

The outline of the paper is as follows. We start with proving necessary and sufficient conditions for uniqueness in the sense of Definition~\ref{def:unique} for various classes of PDEs in Section~\ref{sec:uniqueness-of-pdes}. In Section~\ref{sec:uniqueness-of-odes} we apply these results to specific classes of ODEs, where we utilize the additional structure to deduce stronger statements. Based on the previous sections, Section~\ref{sec:numerical-tests} is devoted to deriving numerical algorithms that allow us to determine whether a function solves a unique PDE. We validate these algorithms in Section~\ref{sec:experiments} by extensive numerical experiments.

\section{Uniqueness of PDEs} \label{sec:uniqueness-of-pdes}

In this section, we aim to derive conditions that guarantee the uniqueness/non-uniqueness of PDEs in the sense of Definition~\ref{def:unique}. Throughout this section we let $u:U\rightarrow\R$ be a differentiable function on the open set $U\subset\R^{m+1}$ and each $g_1,...,g_k:U\rightarrow\R$ either a projection on one of the coordinates, e.g., $g_i(t,x)=x_j$ for $(t,x)\in U$ and $j\in\{1,...,m\}$, any derivative of $u$ that exists, or simply the function $u$. Furthermore, we let $n\in\N\setminus\{0\}$ be such that $\frac{\partial^n u}{\partial^n t}$ exists.

We continue by defining a property over function spaces which turns out to be equivalent to the uniqueness of the PDE in the given function space.

\begin{definition}
    Let $g:U\rightarrow\R^k$, $U\subset\R^{m+1}$ open, be some function and $V$ any set of functions mapping from $\R^k$ to $\R$ including the constant zero-function. Then we say that $g$ is \emph{non-trivially annihilated in} $V$, if there exists some $H\in V\setminus\{0\}$ such that $H\circ g = H(g)=0$. Otherwise, we say that $g$ is \emph{only trivially annihilated in} $V$.
\end{definition}

It is straightforward to prove that the function $g=(g_1,...,g_k)$ is only trivially annihilated in a function class closed under addition and subtraction if and only if $u$ solves a unique PDE described by the given function class for $g_1,...,g_k$ with time derivatives of $n^{th}$-order. 

\begin{proposition} \label{pro:general-uniqueness}
    Define the mapping $g=(g_1,...,g_k):U\rightarrow\mathbb{R}^k$, with $U\subset\R^{m+1}$ open and $g_1,...,g_k$ as in Definition~\ref{def:unique}. Let $V$ be any class of functions mapping from $\R^k$ to $\R$ which is closed under addition and subtraction. Assume there exists $F\in V$ such that $F(g)=\frac{\partial^n u}{\partial^n t}$. Then, $F$ is the unique function in $V$ such that $F(g)=\frac{\partial^n u}{\partial^n t}$ if and only if $g$ is only trivially annihilated in $V$.    
\end{proposition}
\begin{proof}
    Assume now that $g$ is only trivially annihilated in $V$. Let $G\in V$ be a function in $V$ such that 
    $G(g)=F(g)=\frac{\partial^n u}{\partial^n t}$. Then, it follows that $(F-G)(g)=F(g)-G(g)=0$. Since $V$ is closed under subtraction, $F-G\in V$. However, $g$ is only trivially annihilated in $V$ and, therefore, $F-G=0$ must hold. Therefore, $F$ is the unique function in $V$ such that $F(g)=0$.

    Assume now that $g$ is non-trivially annihilated in $V$. Let $H\in V\setminus\{0\}$ be the non-zero function that annihilates $g$, i.e., $H(g)=0$. Set $G=F+H$. Since $V$ is closed under addition, we know that $G\in V$. As $G\neq F$ and $G(g)=F(g)+H(g)=F(g)=\frac{\partial^n u}{\partial^n t}$, $F$ is not unique in $V$.
\end{proof}

Note that any vector space $V$ fulfills the necessary properties for Proposition~\ref{pro:general-uniqueness}, which we use in the rest of the paper, but they are also fulfilled by simpler structures.
In the following subsections, we use Proposition~\ref{pro:general-uniqueness} to deduce uniqueness criteria for several function classes, all summarized in Section~\ref{sec:our-contributions} and Table~\ref{tab:uniqueness-simple-pde}.

\subsection{Linear PDEs} \label{sec:uniqueness-linear-pdes}

The smallest set of PDEs we investigate is the class of linear PDEs, which is well understood theoretically and covers important equations such as the heat and wave equation. In the following, we prove that a function is non-trivially annihilated in the set of linear functions if and only if its coordinates are linearly dependent.

\begin{corollary}[Uniqueness for linear PDEs]\label{cor:uniqueness-for-linear-pdes}
    Define $g=(g_1,...,g_k):U\rightarrow\mathbb{R}^k$, with $U\subset\R^{m+1}$ open and $g_1,...,g_k:U\rightarrow\R$ as in Definition~\ref{def:unique}. Assume there exists at least one linear function $F:\R^k\rightarrow\R$ with $F(g)=\frac{\partial^n u}{\partial^n t}$. Then $F$ is the unique linear function such that $F(g)=\frac{\partial^n u}{\partial^n t}$ if and only if $g_1,...,g_k$ are  linearly independent.
\end{corollary}
\begin{proof}
    The corollary follows immediately from Proposition \ref{pro:general-uniqueness}, since there exists a non-zero linear function $H:\R^k\rightarrow\R$ with $H(g)=0$ if and only if $g_1,...,g_k$ are linearly dependent.
\end{proof}

\subsection{Polynomial and Algebraic PDEs} \label{sec:uniqueness-algebraic-pdes}

The next larger class are polynomial PDEs. Surprisingly, Proposition~\ref{pro:existence-of-algebraic-function-equivalent-polynomial} shows that a function is non-trivially annihilated in the set of polynomial functions if and only if it is non-trivially annihilated in the set of algebraic functions. Thus, all the results in this subsection apply to both classes. 

Similar to, for example, \citet{meshkat2015identifiability,jain2019priori,hong2020global,ovchinnikov2021computing}, who utilized tools from algebra to determine the identifiability for parameter estimation, we develop algebraic criteria for this class of PDEs. We start with defining algebraic functions.

\begin{definition} \label{def:algebraic-function}
    We call a function an \emph{algebraic function} if it solves an equation defined by a non-zero polynomial with real coefficients which is irreducible over $\R$. In other words, an algebraic function is a function $f:U\rightarrow\R$ on an open set $U\subset\R^m$ such that there exists an irreducible non-zero polynomial $p:\R^{m+1}\rightarrow\R$ with $p(x,f(x))=0$ for all $x\in U$.
\end{definition}

Algebraic functions encompass polynomial functions, rational functions, and also non-elementary functions, as can be seen by the Abel-Ruffini Theorem \citep{ruffini1813,abel1813}. Furthermore, the roots and inverse of an algebraic function, if it exists, is algebraic. Thus, this class of PDEs is very general, including the KdV-equation from Example~\ref{ex:kdv-equation}, Helmholtz equation, Burgers' equation, minimal surface equation, and many more.

Polynomials are strongly investigated in the fields of algebra and algebraic geometry and it is therefore beneficial to reformulate the problem in algebraic language. In Appendix \ref{app:fundamentals-of-algebra} we introduce all concepts from algebra which are necessary to follow the proofs in this section. Let us start with algebraic geometry \citep{Kasparian2010}.

\begin{definition} \label{def:algebraic-set}
    Let $ K$ be a field. We call a subset $X\subset K^n$ an \emph{algebraic set} if there exists an ideal $A\subset K[x_1,...,x_n]$ such that $X=\{x\in K^n|f(x)=0,\;\forall f\in A\}$. If $A$ is a principal ideal, i.e., there exists a polynomial $f\in K[x_1,...,x_n]$ such that $A=(f)$, then $X$ is called a \emph{hypersurface} in $ K^n$. 
    
    Given $X\subset K^n$ we define the \emph{correspondence $I$} as
    \begin{equation}
        I(X)\coloneqq\{f\in K[x_1,...,x_n]|f(x)=0,\;\forall x\in X\}.
    \end{equation}
\end{definition}

By definition, these quantities provide two equivalent characterizations of the image $\mathcal{D}=g(U)$ such that $g=(g_1,...,g_k)$ is only trivially annihilated in the set of polynomial functions and in the set of algebraic functions:
\begin{itemize}
    \item[(1)] There exists no algebraic set $X\neq \R^k$ such that $\mathcal{D}\subset X$. In particular, if $g$ is non-trivially annihilated in the set of polynomial functions, then there exists a hypersurface $X$ such that $\mathcal{D}\subset X$.
    \item[(2)] The correspondence $I$ of $\mathcal{D}$ is trivial, i.e., $I(\mathcal{D})=\{0\}$.
\end{itemize}

Additionally, classical algebra provides us with an important definition we will use in the remainder of this section.

\begin{definition} \label{def:algebraic-dependence}
    We call functions $f_1,...,f_q:K^p\rightarrow K$ \emph{algebraically dependent over a field $ K$}, if there exists a polynomial $P\in K[x_1,...,x_q]\backslash \{0\}$ such that $P(f_1(x_1,..,x_p),...,f_q(x_1,..,x_p))=0$. If no such $P$ exists, we call $f_1,...,f_q$ \emph{algebraically independent}.
\end{definition}

This yields a third characterization for $g$ being only trivially annihilated in the set of polynomials:

\begin{itemize}
    \item[(3)] The functions $g_1,...,g_k$ are algebraically independent over $\R$. 
\end{itemize}

We now use this characterization to prove that a function is non-trivially annihilated in the set of polynomial functions if and only if it is non-trivially annihilated in the set of algebraic functions. Thus, all the results in this subsection apply to polynomial and algebraic PDEs.

\begin{proposition} \label{pro:existence-of-algebraic-function-equivalent-polynomial}
    Let $f_1,...,f_k:U\rightarrow\R$ be functions such that there exists a non-zero algebraic function $F:\mathcal{D}\rightarrow\R$ with $F(f_1(x),...,f_k(x))=0$ for each $x\in U$, with $U\subset\R^m$ open and $\mathcal{D}$ the image of $(f_1,...,f_k):U\rightarrow\R^k$. Then, there exists a non-zero polynomial $P:\R^k\rightarrow\R$ such that $P(f_1(x),...,f_k(x))=0$.
\end{proposition}
\begin{proof}
    As $F$ is an algebraic function, there exists an irreducible non-zero polynomial $p:\R^k\times\R\rightarrow\R$ such that $p(y,F(y))=0$ holds for all $y\in\mathcal{D}$. We start with proving that $p(y,0)\neq 0$ for some $y\in\R^k$. 
    
    Towards a contradiction, we assume that $p(y,0)= 0$ for all $y\in\R^k$. Then there exists some polynomial $q:\R^k\times\R\rightarrow\R$ such that $p(y,t)=q(y,t)t$ for all $y\in \R^k$ and $t\in\R$. As $p$ is irreducible and non-zero, $q$ must be constant and non-zero. Thus, $q(y,F(y))F(y)=0$ implies that $F(x)=0$ for all $y\in\mathcal{D}$, which is a contradiction to the assumption. This yields that there exists $y\in\R^k$ such that $p(y,0)\neq 0$.
    
    Obviously, also $p(f_1(x),...,f_k(x),F(f_1(x),...,f_k(x)))=0$ holds for all $x\in U$. Furthermore, $F(f_1(x),...,f_k(x))=0$ yields $p(f_1(x),...,f_k(x),0)=0$. By setting $P(x)=p(x,0)$, we obtain a non-zero polynomial $P:\R^k\rightarrow\R$ with $P(f_1(x),...,f_k(x))=0$.
\end{proof} 

Having established Proposition~\ref{pro:existence-of-algebraic-function-equivalent-polynomial}, we are interested in understanding when a function is non-trivially annihilated in the set of polynomial functions and in the set of algebraic functions in more detail. For continuous differentiable functions $u$, we develop sufficient conditions for uniqueness of analytic PDEs in the next subsection. These can also be used for the smaller class of polynomial PDEs and, thus, also for algebraic PDEs. In the remainder of Subsection~\ref{sec:uniqueness-algebraic-pdes} we restrict $g_1,...,g_k$ to be algebraic functions, as one can prove additional results for polynomial and algebraic PDEs then. 

We start with citing the result that there exist at most $m$ algebraically independent algebraic functions for $m$ unknowns \citep{EhrenborgRota1993}. 

\begin{theorem} \label{thm:algebraic-dependenceof-m+1-polynomials}
    Let $f_1,...,f_{m+1}:U\rightarrow\R$, $U\subset\R^m$ open, be $m+1$ algebraic functions in $m$ variables. Then, $f_1,...,f_{m+1}$ are algebraically dependent. 
\end{theorem}

Contrary to us, \citet{EhrenborgRota1993} state Theorem~\ref{thm:algebraic-dependenceof-m+1-polynomials} for algebraic independence over $\C$. However, for real-valued functions $f_i:\R\rightarrow\R$, algebraic independence over $\R$ is equivalent to algebraic independence over $\C$, as we prove in Lemma~\ref{lem:real-function-algebraic-dependent} in Appendix \ref{app:fundamentals-of-algebra}.

Theorem~\ref{thm:algebraic-dependenceof-m+1-polynomials} shows that if $g_1,..,g_k$ are algebraic functions, we have to ensure that $k\leq m+1$, as otherwise it is certain that $u$ does not solve a unique polynomial PDE for $g=(g_1,...,g_k)$. In the case that we have at least as many unknown variables as functions, it is possible to use the Jacobian criterion \citep{EhrenborgRota1993}. This provides an equivalent condition for a function to be non-trivially annihilated in the set of polynomial functions and in the set of algebraic functions. Again, we must apply Lemma~\ref{lem:real-function-algebraic-dependent} from Appendix \ref{app:fundamentals-of-algebra} to obtain independence over $\R$.

\begin{theorem}[Jacobian criterion for algebraic independence]\label{thm:Jacobian-criterion-for-algebraic-independence}
    Define $g=(g_1,...,g_k):U\rightarrow\mathbb{R}^k$, with $U\subset\R^{m+1}$ open and $g_1,...,g_k:U\rightarrow\R$ as in Definition~\ref{def:unique}. Furthermore, assume that $g_1,..,g_k$ are algebraic functions and $k\leq m+1$. Then, $g_1,...,g_k$ are algebraically independent over $\R$ and, thus, $g$ is only trivially annihilated in the set of polynomial functions and in the set of algebraic functions, if and only if there exists one point $(t,x)\in U$ with $rank(J_g(t,x))= k$.
\end{theorem}

An interesting fact about algebraic functions is that the derivative of an algebraic function is again algebraic. This means that the condition $g_1,...,g_k$ being algebraic from Theorem~\ref{thm:algebraic-dependenceof-m+1-polynomials} and~\ref{thm:Jacobian-criterion-for-algebraic-independence} follows from $u$ being algebraic.

\begin{proposition} \label{pro:derivative-of-algebraic-function}
    Let $u:\R^m\rightarrow\R$ be an algebraic function and $\alpha\in\N^m$ be a multi-index such that $u_\alpha$ exists and is continuous. Then, $u_\alpha$ is an algebraic function.
\end{proposition}
\begin{proof}
    We only show the statement for $|\alpha|=1$, as the case $|\alpha|>1$ follows by induction. 
    
    In the case $|\alpha|=1$, there exists $1\leq i\leq m$ such that $u_\alpha=u_i$. Let $x=(x_1,...,x_m)$. As $u$ is an algebraic function, we know that there exists a non-zero polynomial $p\in\R[x,t]$ such that 
    \begin{equation} \label{eq:proof-polynomial-of-algebraic-function}
        p(x,u(x))=0.
    \end{equation}
    Let $p$ be a polynomial with minimal degree such that \eqref{eq:proof-polynomial-of-algebraic-function} holds. Differentiating \eqref{eq:proof-polynomial-of-algebraic-function} with respect to $x_i$ yields 
    \begin{equation} \label{eq:proof-derivative-of-polynomial-of-algebraic-function}
        p_{x_i}(x,u(x))+p_t(x,u(x))u_i(x)=0.
    \end{equation}
    Now, define $f(x,t)=p_{x_i}(x,u(x))+p_t(x,u(x))t$. In Appendix \ref{app:fundamentals-of-algebra} we introduce the concept of algebraic functions as the algebraic closure of a rational function field and, thus, we know that the powers and sums of algebraic functions are again algebraic functions. It follows that $f$ is an algebraic function.
    
    Since we chose $p$ as the non-zero polynomial with minimal degree such that $p(x,u(x))=0$, we know that $p_{x_i}(x,u(x))$ and $p_t(x,u(x))$ are both non-zero and, therefore, $f$ is non-zero. Furthermore, \eqref{eq:proof-derivative-of-polynomial-of-algebraic-function} yields that $f(x,u_i(x))=0$. From Proposition~\ref{pro:existence-of-algebraic-function-equivalent-polynomial} it then follows that $u_i$ is an algebraic function.
\end{proof}

\subsection{Analytic PDEs} \label{sec:uniqueness-analytic-PDEs}
The logical extension of polynomials are analytic functions. This class is particularly large and encompasses most of the relevant PDEs, including\textemdash{}in addition to the subset of linear and polynomial PDEs\textemdash{}Liouville's equation, Zeldovich–Frank-Kamenetskii equation, Calogero–Degasperis–Fokas equation, and Josephson equations.

\begin{definition}[Analytic function \citep{Krantz1992APO}]
    A function $F\in C^\infty(X)$ with $X\subset\R^k$ open is called an \emph{analytic function} if at each point $x\in X$ there exists an open set $V\subset X$ with $x\in V$ such that the Taylor series of $F$ converges pointwise to $F$ on $V$. We denote the set of analytic functions on $X$ as $C^\omega(X)$.
\end{definition}

$C^\omega$ includes polynomials but also exponential and trigonometric functions. Furthermore, sums, products, and compositions of analytic functions are analytic functions as well, as are reciprocals of non-zero analytic functions and the inverse of an analytic function with a non-zero derivative. We remark that there are functions that are algebraic but not analytic such as $x^{1/3}$. However, there also exist functions that are analytic but not algebraic such as $\exp(x)$. More information on this broad set of functions can be found in \citet{Krantz1992APO}.

The analog of algebraic sets in algebraic geometry are analytic sets in analytic geometry \citep{Chirka1989}. 

\begin{definition}[Analytic set] \label{def:analytic-set}
    A set $A\subset\R^n$ is called an \emph{analytic set} if for each $a\in A$ there exists a neighbourhood $V$ such that $A\cap V = \{x\in\R^n|f_1(x)=...=f_m(x)=0\}$ for some analytic functions $f_1,...,f_m$.
\end{definition}

Hence, we can conclude that the image $\mathcal{D}$ of $g$ is a subset of an analytic set $A\neq\R^k$, if $g$ is non-trivially annihilated in $C^\omega$. However, we cannot conclude that $g$ is non-trivially annihilated in $C^\omega$, if its image $\mathcal{D}$ is a subset of an analytic set $A\neq\R^k$, because of the locality in the definition of analytic sets. Thus, the following definition from \citet{Acquistapace2017} is more fitting for our purposes.

\begin{definition}[C-analytic set] \label{def:c-analytic-set}
    A set $A\subset\R^n$ is called a \emph{C-analytic set}, if $A = \{x\in\R^n|f_1(x)=...=f_m(x)=0\}$ for some analytic functions $f_1,...,f_m$.
\end{definition}

For this class of sets, we obtain equivalence: A function $g$ is non-trivially annihilated in $C^\omega$ if and only if its image $\mathcal{D}$ is a subset of a C-analytic set $A\neq\R^k$.

We consider now a sufficient condition for the uniqueness of analytic functions which is proven in \citet{Mityagin}:

\begin{proposition} \label{pro:measure-condition-for-uniqueness-of-analytic-functions}
    Let $F:\R^k\rightarrow\R$ be an analytic function and $\mathcal{D}\subset\R^k$ a set with $\lambda^k(\mathcal{D})>0$, where $\lambda^k$ is the $k$-dimensional Lebesgue-measure. Then $F|_\mathcal{D}=0$ implies $F=0$.
\end{proposition}

This implies immediately the following corollary.

\begin{corollary}[Measure criterion for analytic functions] \label{cor:measure-condition-for-uniqueness-of-analytic-functions}
Define the function $g=(g_1,...,g_k):U\rightarrow\mathbb{R}^k$, with $U\subset\R^{m+1}$ open, $g_1,...,g_k:U\rightarrow\R$ as in Definition~\ref{def:unique} and set $\mathcal{D}\coloneqq g(U)$. If $\lambda^k(\mathcal{D})>0$, then $g$ is only trivially annihilated in $C^\omega$.
\end{corollary}

Corollary~\ref{cor:measure-condition-for-uniqueness-of-analytic-functions} shows the uniqueness of the analytic PDE provided that the image of $g$ is a set with positive measure. The image of a differentiable function, which maps from a lower-dimensional to a higher-dimensional space, is always a null set \citep{rudin1987}. Therefore, $\lambda^k(\mathcal{D})>0$ can only be true if $m+1\geq k$. In this case, the uniqueness of the PDE can be assessed by investigating the Jacobian matrix of $g$, as can be seen in the following theorem. This criterion is of high interest since it can be checked numerically, see Section \ref{sec:numerical-tests}.

\begin{theorem}[Jacobian criterion for analytic functions]\label{thm:Jacobian-condition-for-uniqueness-of-analytic-functions}
    Define the function $g=(g_1,...,g_k):U\rightarrow\mathbb{R}^k$, with $U\subset\R^{m+1}$ open and $g_1,...,g_k:U\rightarrow\R$ as in Definition~\ref{def:unique}. Furthermore, assume that $g_1,...,g_k$ are continuously differentiable. If there exists one point $(t,x)\in U$ with $rank(J_g(t,x))= k$, then $g$ is only trivially annihilated in $C^\omega$.
\end{theorem}
\begin{proof} 
    Let $(y^0)\in U$ such that $rank(J_g(y^0))= k$. Then there exist $k$ independent columns of $J_g(y^0)$ and, without loss of generality, we can achieve $\det\left(\left(\frac{dg(y^0)}{dy_i}\right)_{i=1,...,k}\right)\neq 0$ by reordering the components of $y^0$. 
    
    Let $V\coloneqq U\cap\R^k$. Define the function $\Tilde{g}:V\rightarrow\R^k$ by $\Tilde{g}(y_1,...,y_k)=g(y_1,...,y_k,y_{k+1}^0,...,y_{m+1}^0)$. By construction, we obtain $\det\left(J_{\Tilde{g}}((y_1^0,...,y_k^0))\right)\neq0$ and, thus, the inverse function theorem yields that there exist open sets $A\subset V$ and $B\subset\R^k$ such that $\Tilde{g}(A)=B$. This means that $B$ is in the image of $\Tilde{g}$ and, therefore, also of $g$, i.e., $B\subset\mathcal{D}$. As $B$ is non-empty and open, this yields $0<\lambda^k(B)\leq\lambda^k(\mathcal{D})$.
\end{proof}

Theorem~\ref{thm:Jacobian-condition-for-uniqueness-of-analytic-functions} is seemingly similar to Theorem~\ref{thm:Jacobian-criterion-for-algebraic-independence}. However, it extends Theorem~\ref{thm:Jacobian-criterion-for-algebraic-independence} in two ways. First of all, it is a condition for the bigger class of analytic differential equations. More importantly, it is not restricted to algebraic functions $g_1,...,g_k$, but only requires them to be continuously differentiable. The generality comes at the loss of the equivalence between the full rank of the Jacobian and the uniqueness of the differential equation. 

The question now is whether $\lambda(\mathcal{D})>0$ is also necessary for the uniqueness of the analytic PDE. It turns out that this is not the case, at least if we disregard the structure of $g$. This follows from Lemma~\ref{lem:zero-measure-but-U-property} which is proven in \citet{Neelon}. Notice, that \citet{Neelon} proves this result only for functions $f$ with $f(0)=0$. This can be circumvented in Lemma~\ref{lem:zero-measure-but-U-property} below by applying the original statement to $\Tilde{f}(x)=f(x)-c$ and $\Tilde{g}(x,y)=g(x,y+c)$ to include general functions $f$ with $f(0)=c$. This approach succeeds since $f$ is an analytic function if and only if $\Tilde{f}$ is one.

\begin{lemma} \label{lem:zero-measure-but-U-property}
    Let $f:\R^m\rightarrow\R$ be in $C^\infty\setminus C^\omega$. Then $g=(x,f(x))$ is only trivially annihilated in $C^\omega$.
\end{lemma}

Let $g:\R^m\rightarrow\R^{m+1}$ be defined as in Lemma~\ref{lem:zero-measure-but-U-property}. As $g$ is differentiable, we know that $\lambda^{m+1}(g(\R^m))=0$ \citep{rudin1987}. Therefore, Lemma~\ref{lem:zero-measure-but-U-property} yields examples for which $\lambda^k(\mathcal{D})=0$, but the function $g$ is only trivially annihilated in $C^\omega$. This shows that $\lambda^k(\mathcal{D})>0$ is only a sufficient but not a necessary condition, for $g$ being only trivially annihilated by $C^\omega$.

In Appendix \ref{app:analytic-independence} we discuss the problem that occurs when applying the matroid framework \citep{EhrenborgRota1993} to analytic functions and analytic PDEs. This and Lemma~\ref{lem:zero-measure-but-U-property} reveal the severe difficulty of achieving stronger uniqueness results for analytic functions. 

\subsection{Continuous and Smooth PDEs} \label{sec:uniqueness-smooth-pdes}
As the last function class, we consider the class of $C^p$ functions, for any $0\leq p\leq\infty$. It turns out that for this class we can rarely achieve uniqueness. For this, we first recall Proposition~2.3.4 from \citet{Krantz1992APO}.

\begin{proposition} \label{pro:existence-c-infty}
    Let $E\subset\R^k$ be any closed set. Then there exists a function $H\in C^\infty(\R^k)$ such that $H^{-1}(\{0\})=E$.
\end{proposition}
This allows us to derive the following statement:
\begin{theorem}[Uniqueness for $C^p$ functions]\label{thm:uniqueness-c-infty-functions}
    Define $g=(u_{\alpha^1},...,u_{\alpha^k}):U\rightarrow\mathbb{R}^k$, with $U\subset\R^{m+1}$ open, $g_1,...,g_k:U\rightarrow\R$ as in Definition~\ref{def:unique} and set $\mathcal{D}=g(U)$. Then the function $g$ is only trivially annihilated in $C^p$, for any $0\leq p\leq \infty$, if and only if the closure of $\mathcal{D}$ is equal to $\R^k$.
\end{theorem}
\begin{proof}
    "$\Rightarrow$" Assume that the closure $\overline{\mathcal{D}}$ does not equal $\R^k$. Then, Proposition~\ref{pro:existence-c-infty} yields that there exists $H\in C^\infty$ such that
    $H^{-1}(\{0\})=\overline{\mathcal{D}}$. Therefore,
    $H(g)=0$ holds, but $H\neq0$ on
    $\R^k\setminus\overline{\mathcal{D}}$. Thus, $g$ is non-trivially annihilated in $C^\infty$ and, therefore, also in any $C^p$ with $0\leq p\leq \infty$.
    
    "$\Leftarrow$" Assume now that $\overline{\mathcal{D}}=\R^k$. Let $H\in C^0$ be such that $H|_\mathcal{D}=0$. 
    Since $H$ is continuous, we obtain that $H$ is 0 on the entire space
    $\R^k$. Thus, $g$ is only trivially annihilated in $C^0$ and, consequently, also in $C^p$, for any $0\leq p\leq\infty$.
\end{proof}

This shows that for $C^p$, for any $0\leq p\leq\infty$, we can never achieve uniqueness unless the image of $g$ is dense in $\R^k$. This shows the significance of the ambiguity of PDEs for physical law learning and also motivates the focus on smaller function classes.

Note that, if we allow $F$ to be discontinuous, surjectivity of $g$ is necessary and sufficient for ensuring uniqueness in the sense of Definition~\ref{def:unique}.

\section{Uniqueness of ODEs} \label{sec:uniqueness-of-odes}

In the previous section, we investigated the uniqueness of PDEs of the form $\frac{\partial^n u}{\partial^n t}=F(g_1,...,g_k)$, $1\leq i \leq k$, and allowed $g_1,...,g_k$ to be either a projection on a single coordinate, any derivative of $u$ that exists or the function $u$. This enables us to apply our developed theoretical framework in this section to ODEs as special cases. Indeed, ODEs are of significant importance for applications, for example, for electrical circuits, PID-controller, and any system derived from Euler-Langrage equations \citep{chicone2006}.

In each of the following subsections, we investigate a specific type of ODEs. Our analysis will show that additional structures will lead to more refined necessary and sufficient conditions than merely restricting the previous case to general ODEs. Throughout this section $u:U\rightarrow\R$ will be a differentiable function on $U\subset\R$ open. We denote from now on $u_t=\frac{\partial u}{\partial t}$, $u_{tt}=\frac{\partial^2 u}{\partial^2 t}$, and $u_{ttt}=\frac{\partial^3 u}{\partial^3 t}$.

\subsection{Autonomous ODEs}

We start with autonomous ODEs, i.e., ODEs that do not depend directly on $t$, but only indirectly through $u$ and its derivatives.

\subsubsection{First Order ODEs $u_t=F(u)$}

As long as $u\in C^1(U)$, with $U\subset\R$ open, is not a constant function, we obtain that $\lambda(\mathcal{D})>0$, where $\mathcal{D}=u(U)$. Therefore, $u$ is only trivially annihilated in the sets of linear functions, polynomials, or analytic functions. Thus, $u$ solves a unique ODE described by these function classes for $g=(u)$ with first-order time derivative by Proposition~\ref{pro:measure-condition-for-uniqueness-of-analytic-functions}.

Uniqueness in the sense of Definition~\ref{def:unique} for $C^p$, $0\leq p \leq \infty$ follows from Theorem~\ref{thm:uniqueness-c-infty-functions} if and only if the image of $u$ is dense in $\R$. Density in $\R$, however, is equivalent to $u$ being surjective, since we assume $u$ to be differentiable and, thus, continuous.

\subsubsection{Second Order ODEs $u_{tt}=F(u,u_t)$}

By Corollary~\ref{cor:uniqueness-for-linear-pdes}, the function $g=(u,u_t):U\rightarrow\R^2$, $U\subset\R$ open, is non-trivially annihilated in the set of linear functions if and only if $u$ and $u_t$ are linearly independent. This yields the following proposition:

\begin{proposition} \label{pro:second-order-linear-ode}
    Let $u:U\rightarrow\R$, $U\subset\R$ open, and $F:\R^2\rightarrow\R$ linear such that $u_{tt}=F(u,u_t)$. Then $F$ is the unique linear function with $u_{tt}=F(u,u_t)$ if and only if $u$ is neither a constant nor an exponential function.
\end{proposition}
\begin{proof}

    "$\Rightarrow$" If $u$ is constant or an exponential function, then $u$ and $u_t$ are linearly dependent and, thus, $F$ is not unique.
    
    "$\Leftarrow$" We prove that if $F$ is not unique, then $u$ is a constant or exponential function. For this, assume that $F$ is not unique. The functions $u$ and $u_t$ are then linearly dependent by Corollary~\ref{cor:uniqueness-for-linear-pdes}. Assuming $u$ is not constant, there exists $\lambda\in\R\setminus\{0\}$ such that $u=\lambda u_t$. This ODE is uniquely solved by the exponential function, which concludes our proof.
\end{proof}

If $u$ is an algebraic function, Theorem~\ref{thm:algebraic-dependenceof-m+1-polynomials} implies that $F$ does not describe the unique polynomial PDE for $g=(u,u_x)$. As $u$ is differentiable, we obtain that $\lambda^2(\mathcal{D})=0$, where $\mathcal{D}$ is the image of $g=(u,u_t)$. This allows to conclude that Proposition~\ref{pro:measure-condition-for-uniqueness-of-analytic-functions} and, consequently, Theorem~\ref{thm:Jacobian-condition-for-uniqueness-of-analytic-functions} will not yield any information. This also implies that $\mathcal{D}$ is not dense in the image space and, thus, $g$ is non-trivially annihilated in $C^p$.

\subsubsection{Third Order ODEs $u_{ttt}=F(u,u_t,u_{tt})$}

Uniqueness for linear ODEs in $u$, $u_t$, and $u_{tt}$ is equivalent to the independence of $u$, $u_t$, and $u_{tt}$, which results in the following proposition.

\begin{proposition}
    Let $u:U\rightarrow\R$, with $U\subset\R$ open, and let $F$ be linear such that $u_{tt}=F(u,u_t)$. Then $F$ is the unique linear function with $u_{ttt}=F(u,u_t, u_{tt})$ if and only if $u$ is none of the following functions:
    \begin{itemize}
        \item[(1)] a constant function,
        \item[(2)] an exponential function,
        \item[(3)] a linear combination of two exponential functions,
        \item[(4)] a linear combination of $t\mapsto\exp(\mu t)$ and $t\mapsto t\exp(\mu t)$ for any $\mu\in\R$ and
        \item[(5)] a linear combination of $t\mapsto\exp(\mu t)\cos(\omega t)$ and  $t\mapsto\exp(\mu t)\sin(\omega t)$, for any $\mu,\omega\in\R$.
    \end{itemize}
\end{proposition}
\begin{proof}
$F$ is the unique linear function with $u_{ttt}=F(u,u_t, u_{tt})$ if and only if $u$, $u_t$ and $u_{tt}$ are  linearly independent. Thus, $F$ is not unique if and only if there exists a non-zero $\lambda\in\R^3$ such that $\lambda_1u+\lambda_2u_t+\lambda_3u_{tt}=0$. The case $\lambda_3=0$ is equivalent to the setting in Proposition~\ref{pro:second-order-linear-ode} and results in the first two options. The case $\lambda_3\neq0$ yields a second-order linear ODE with constant coefficients and it is known that those can only be solved by the functions of items (3) to (5) \citep{thompson2013ordinary}.
\end{proof}

For third-order polynomial, algebraic, analytic, smooth, and continuous ODEs the uniqueness conditions coincide with those in the previous case $u_{tt}=F(u,u_t)$. 

\subsection{Non-Autonomous ODEs}
We are now considering non-autonomous ODEs. The theory developed in Section~\ref{sec:uniqueness-of-pdes} is still applicable as we allowed $g_1,...,g_k$ to be a projection on a single coordinate, i.e., we can now choose $g_1(t)=t$ for all $t\in U$. It is important to note, however, that it is not meaningful to consider functions $F$ which are, e.g., linear in all coordinates as usually ODEs are only linear in $u$ and its derivatives, but not in $t$.

\subsubsection{First order ODEs $u_t=F(t,u)$} 

We start by assuming that $F$ is linear in $u:U\rightarrow\R$ and continuous in $t\in U$, with $U\subset\R$ open. This means that we can write $F(t,u(t))=\lambda(t)u(t)$ for some continuous function $\lambda:U\rightarrow\R$. For ODEs of this type, we can prove the following proposition. 

\begin{proposition}
    Let $u:U\rightarrow\R$, $U\subset\R$ open, be any differentiable function and let $F:\R\times\R\rightarrow\R$ be continuous in the first component and linear in the second component with $F(t,u)=u_t$. Then, $F$ is the only function which is continuous in the first and linear in the second component if and only if $u$ has only isolated zeros.
\end{proposition}
\begin{proof}

     Let $F(t,u(t))=\lambda(t)u(t)$, for $\lambda:U\rightarrow\R$ continuous and $t\in U$. As the class of functions which are continuous in the first component and linear in the second is closed under addition and subtraction, we can apply Proposition~\ref{pro:general-uniqueness} and, thus, only have to check if $g(t)=(t,u(t))$ is non-trivially annihilated in this function class. This means that we have to investigate whether there exists a nonzero function $H$ in this function class satisfying $H(t,u(t))=0$ for all $t\in U$. 
    
    "$\Rightarrow$" We show that if $u$ has a zero which is not isolated, then $F$ is not unique. For this, let $\epsilon>0$ be such that $u(B_\epsilon(t_0))=0$ holds, with $B_\epsilon(t_0)\coloneqq\{t\in\R:|t-t_0|<\epsilon\}\subset U$. Defining $\phi:\R\rightarrow\R$ as a bump function with the closure of $B_\epsilon(t_0)$ as support and $H:\R^2\rightarrow\R$ by $H(t,s)=\phi(t)s$ yields $H(t,u(t))=\phi(t)u(t)=0$ for all $t\in U$. Since $H$ is nonzero, continuous in $t$, and linear in $s$, the function $g(t)=(t,u(t))$ is non-trivially annihilated in the set of functions which are continuous in the first and linear in the second component.
    
    "$\Leftarrow$" Let $H(t,s)\coloneqq\phi(t)s$ be such that 
    \begin{equation} \label{eq:proof-linear-non-autonomous-first-order-ODE}
        H(t,u(t))=\phi(t)u(t)=0
    \end{equation}
    for all $t\in U$.
    Assume that $H$ is not the zero function. Then there exists $t_0\in U$ such that $\phi(t_0)\neq0$. \eqref{eq:proof-linear-non-autonomous-first-order-ODE} yields that $\phi(t_0)\neq0$ implies $u(t_0)=0$. As the zeros of $u$ are isolated, there exists a ball $B_\epsilon(t_0)\subset U$ with radius $\epsilon>0$ around $t_0$ such that $u(t)\neq0$ for all $t\in B_\epsilon(t_0)\setminus\{t_0\}$. By \eqref{eq:proof-linear-non-autonomous-first-order-ODE}, $\phi(t)=0$ holds for all $t\in B_\epsilon(t_0)\setminus\{0\}$ and, by continuity, also at $t_0$. This is a contradiction to $\phi(t_0)\neq0$. Thus, $H$ must be the zero function and, by Proposition~\ref{pro:general-uniqueness}, $F$ is unique in its class.
\end{proof}

Now, consider the polynomial ODE case, i.e., assume that $F(t,u_t)$ is continuous in $t$ and polynomial in $u_t$. If we assume no additional constraints on this class of PDEs, we immediately obtain that $g(t)=(t,u(t))$ is non-trivially annihilated, since $H(t,s)=u(t)-s$ is non zero, continuous in $t$, polynomial in $s$ and $H(t,u(t))=0$ for all $t\in U$. A possible constraint is that $F(t,s)$ has to be polynomial in $t$ and $s$. By definition, $g$ is then non-trivially annihilated if and only if $u(t)$ is an algebraic function. Applying Proposition~\ref{pro:existence-of-algebraic-function-equivalent-polynomial} extends this to the case of algebraic functions $F$.

For $F(t,s)$ continuous in $t$ and analytic in $s$, uniqueness in the sense of Definition~\ref{def:unique} is impossible as well. Thus, let us now assume that $F$ is analytic in both $t$ and $s$. In the following, we prove that for this class $g(t)=(t,u(t))$ is non-trivially annihilated if and only if $u$ is analytic.

\begin{proposition}
    The function $g(t)=(t,u(t))$ is non-trivially annihilated in $C^\omega$ if and only if $u$ is analytic.
\end{proposition}
\begin{proof}
"$\Leftarrow$" Let $u$ be an analytic function. Then, $H(t,s)=u(t)-s$ is a non-zero analytic function fulfilling $H(t,u(t))=0$ for all $t\in U$. Thus, $g(t)=(t,u(t))$ is non-trivially annihilated in $C^\omega$. 

"$\Rightarrow$" This direction follows directly from Lemma~\ref{lem:zero-measure-but-U-property}.    
\end{proof}

\subsubsection{Second order ODEs $u_{tt}=F(t,u,u_t)$} \label{sec:non-autonomous-2nd-order-ode}

Similar to the last section we start with assuming that $F(t,u,u_t)=\lambda(t)u(t)+\mu(t)u_t(t)$ is linear in $u$ and $u_t$ and continuous in $t$, i.e., $\lambda$ and $\mu$ are arbitrary continuous functions. Unfortunately, uniqueness in the sense of Definition~\ref{def:unique} of $F$ can never be achieved in this case.

\begin{proposition}
    Let $u:U\rightarrow\R$, $U\subset\R$ open, be any differentiable function and let $F:\R^3\rightarrow\R$ be continuous in the first component and linear in the second and third component with $F(t,u,u_t)=u_{tt}$. Then $F$ is not the unique function which is continuous in the first and linear in the second and third components and fulfills $F(t,u,u_t)=u_{tt}$.
\end{proposition}
\begin{proof}
    Set $\Tilde{\lambda}(t)=u_t(t)$ and $\Tilde{\mu}(t)=-u(t)$ for all $t\in U$ to define $H:U\times\R^2\rightarrow\R^3,H(t,y,z)\coloneqq\Tilde{\lambda}(t)y+\Tilde{\mu}(t)z=u_t(t)y-u(t)y$. Notice that $H$ is continuous in the first component since $u$ is twice differentiable. Furthermore, $H$ is linear in the second and third and fulfills $H(t,u,u_t)=0$ for any $t$. Thus, $g:U\rightarrow\R^3,g(t)\coloneqq(t,u(t),u_t(t))$ is non-trivially annihilated in the class of functions, which are continuous in the first and linear in the second and third component.
\end{proof}

Consequently, uniqueness of $F$ can also never be achieved for functions with polynomial, algebraic, analytic, smooth, or continuous second and third components.

\section{Algorithms} \label{sec:numerical-tests}
In the last two sections, we established theoretical criteria for the uniqueness problem of ODE and PDE learning. One immediate application is to deduce algorithms from the theory to enable practitioners to check numerically if they are facing ambiguity in the ODE or PDE. These algorithms will determine for given multi-indices $\alpha^1,...,\alpha^k$, whether there exists more than one function $F$ in a specific function class for which $F(u_{\alpha^1},...,u_{\alpha^k})=u_t$ holds.

\subsection{Linear PDEs} \label{sec:numerical-tests-linear-pdes}

Corollary~\ref{cor:uniqueness-for-linear-pdes} shows that we have to check for linear dependency of functions in the case of linear PDEs. Linear dependence can be determined by applying SVD and comparing the least singular value with some threshold \citep{Press20007}. This is displayed in Algorithm~\ref{alg:franco} as Feature Rank Computation (FRanCo).

\begin{algorithm}
\caption{Feature Rank Computation (FRanCo)} \label{alg:franco}
\algorithmInput{$(g(t^i,x^j))_{ij}=((u_{\alpha^1}(t^i,x^j),...,u_{\alpha^k}(t^i,x^j))_{ij}$, threshold $\delta>0$}
\algorithmOutput{Boolean value indicating if $g$ is non-trivially annihilated in the set of linear functions}
\begin{algorithmic}[1]
\State $least\_sv\leftarrow$ smallest singular value of $(g(t^i,x^j))_{ij})$
\If{$least\_sv < \delta$}
\State return $False$ \Comment{a small singular value indicates non-uniqueness}
\Else
\State return $True$ \Comment{large singular values indicate uniqueness}
\EndIf
\end{algorithmic}
\end{algorithm}

In general, one has to deal carefully with the errors introduced through numerical derivatives. Therefore, we gradually increase the order of the finite differences and check if the lowest singular values converge to 0 exponentially fast. This is expected to happen if the exact matrix does not have full rank, since higher order approximations have higher order residual terms \citep{Burden2015}. Therefore, we propose applying Stable Feature Rank Computation (S-FRanCo), described in Algorithm~\ref{alg:s-franco}, and assess if the least singular value decreases exponentially for higher finite differences orders. The advantage of this approach becomes clear in Section~\ref{sec:experiments}. S-FRanCo takes a method $features$ as input. The method $features$ outputs the desired feature matrix for given derivatives and is needed as an input to S-FRanCo, to determine the feature matrix of which the rank is supposed to be computed. Its explicit application becomes apparent in the next two subsections. For linear PDEs, it should simply concatenate the derivatives to a linear feature matrix, i.e., for the inputs $u_{\alpha^1}(t^i,x^j)_{ij},...,u_{\alpha^k}(t^i,x^j)_{ij}$ $features$ should return 
\begin{equation}
    U = \left(\begin{array}{ccc}
    | & \hdots & |\\
    u_{\alpha^1}(t^i,x^j)_{ij} & \hdots & u_{\alpha^k}(t^i,x^j)_{ij} \\
    | &  \hdots & |
    \end{array}\right).    
\end{equation}

\begin{algorithm}
\caption{Stable Feature Rank Computation (S-FRanCo)} \label{alg:s-franco}
\algorithmInput{$(u(t^i,x^j))_{ij}$, $\alpha^1,...,\alpha^k$ multi-indices, $features$ method which takes the derivatives as input and outputs the desired feature matrix}
\algorithmOutput{List of the lowest singular value for each finite-difference order of the matrix consisting of monomials of the vectors
$(u_{\alpha^1}(t^i,x^j)_{ij},...,u_{\alpha^k}(t^i,x^j)_{ij})$}
\begin{algorithmic}[1]
    \State Let $ls\_sv$ be an empty list
    \For{$l=2$ to $d$}
    \For{$m=1$ to $k$}
    \State Let $(u_{\alpha^m}(t^i,x^j))_{ij}$ be the derivative computed by finite differences of $l^{th}$ order for multi-index $\alpha^m$
    \EndFor
    \State $(g(t^i,x^j))_{ij}\leftarrow features((u_{\alpha^1}(t^i,x^j)_{ij},...,u_{\alpha^k}(t^i,x^j)_{ij}))$ \Comment{construct the feature library}
    \State $least\_sv\leftarrow$ smallest singular value of $(g(t^i,x^j))_{ij})$ 
    \State $ls\_sv.append(least\_sv)$
    \EndFor
    \State return $ls\_sv$
\end{algorithmic}
\end{algorithm}

\subsection{Polynomial and Algebraic PDEs} \label{sec:numerical-experiment-polynomials}

If we assume that $u$ is an algebraic function, Theorem~\ref{thm:Jacobian-criterion-for-algebraic-independence} yields that the Jacobian has full rank if and only if $g$ is only trivially annihilated in the sets of polynomials or algebraic functions. We can check the Jacobian rank for a selected number of points with Jacobian Rank Computation (JRC), as described in Algorithm~\ref{alg:Jacobian-criterion}. As numerical errors also influence the Jacobian, we propose to use two different finite-difference orders, a small one and a large one. Afterward, one has to determine if there exists at least one point for which the least singular value decreased significantly. This can be visualized by a heat map as shown in Figure~\ref{fig:Jacobian-polynomial-ambiguous-algebraic-function} in Section~\ref{sec:experiments-polynomials-algebraic-function}.

\begin{algorithm}
\caption{Jacobian Rank Computation (JRC)} \label{alg:Jacobian-criterion}
\algorithmInput{$(u(t^i,x^j))_{ij}$, $\alpha^1,...,\alpha^k$ multi-indices,  $d_1$ small finite-differences order, $d_2$ large finite-differences order, $A$ index set determining the data points for which the Jacobian is computed}
\algorithmOutput{For both finite-difference orders a list of the lowest singular values of the Jacobian at selected points of the matrix
$((u_{\alpha^1}(t^i,x^j),...,u_{\alpha^k}(t^i,x^j))_{ij}$}
\begin{algorithmic}[1]
\State Let $ls\_sv_1$ and $ls\_sv_2$ be empty lists
\For{$p=1$ to 2}
    \For{$m=1$ to $k$} \Comment{Compute the function $g$}
        \State Let $(u_{\alpha^m}(t^i,x^j)_{ij})$ be the derivative computed by finite differences of $d_p^{th}$ order for multi-index $\alpha^m$
    \EndFor
    \State $(g(t^i,x^j))_{ij}\leftarrow((u_{\alpha^1}(t^i,x^j),...,u_{\alpha^k}(t^i,x^j))_{ij}$
    \For{$(i,j)\in A$} \Comment{Compute the least singular value of the Jacobian of $g$ at selected points}
        \State $jacobian\leftarrow$ the Jacobian of $((g(t^i,x^j))_{ij}$ computed using finite differences of $d_p^{th}$ order
        \State $least\_sv\leftarrow$ smallest singular value of $(g(t^i,x^j))_{ij})$ 
        \State $ls\_sv_p.append(least\_sv)$
    \EndFor
\EndFor
\State return $ls\_sv$
\end{algorithmic}
\end{algorithm}

In case one can not assume that $u$ is algebraic, we can still use the Jacobian criterion for analytic functions from Theorem~\ref{thm:Jacobian-condition-for-uniqueness-of-analytic-functions}, since uniqueness for analytic PDEs implies uniqueness for polynomial PDEs. Therefore, a full-rank Jacobian in at least one data point yields the uniqueness of the PDE. However, the Jacobian criterion for analytic functions cannot be used to prove non-uniqueness and, thus, we need a different algorithm for the case that the Jacobian has nowhere full rank.

Thus, if one does not want to assume that $u$ is algebraic, we propose the following approach:
\begin{itemize}
    \item[(1)] Check for uniqueness in the larger space of analytic PDEs using the Jacobian criterion from Theorem~\ref{thm:Jacobian-condition-for-uniqueness-of-analytic-functions} by applying JRC. If the Jacobian has full rank for at least one point $(t,x)\in\R^{m+1}$, Theorem~\ref{thm:Jacobian-condition-for-uniqueness-of-analytic-functions} yields that the PDE is not unique. Note that Theorem~\ref{thm:Jacobian-condition-for-uniqueness-of-analytic-functions} is a consequence of Proposition~\ref{pro:measure-condition-for-uniqueness-of-analytic-functions}. Thus, it can only be meaningfully used if there are fewer derivatives $u_{\alpha^1},...,u_{\alpha^k}$ than input variables $(t,x_1,...,x_m)$, i.e., if $k\leq m+1$, as mentioned in Section~\ref{sec:uniqueness-analytic-PDEs}. Therefore, JRC is only useful in this case.

    \item[(2)] If Theorem~\ref{thm:Jacobian-condition-for-uniqueness-of-analytic-functions} cannot be applied, check uniqueness for polynomials with degree at most $p\in\N$ by constructing a feature matrix from the monomials of the derivatives and determining the rank of this matrix. If this matrix has full rank, the PDE is unique among polynomials with degree $p$. Otherwise, we know that it is not unique. This can be determined robustly by S-FranCo again, by defining the method $features$ such that it returns all monomials of the input vectors with degree at most $p$. We can then assess the rank of the matrices by checking for exponential decay of the lowest singular value.
\end{itemize}

\subsection{Analytic PDEs} \label{sec:numerical-tests-analytic-functions}

Following Theorem~\ref{thm:Jacobian-condition-for-uniqueness-of-analytic-functions}, JRC can also be applied to show that an analytic PDE is unique. Unfortunately, Theorem~\ref{thm:Jacobian-condition-for-uniqueness-of-analytic-functions} yields only a sufficient condition for uniqueness and, thus, we gain no information if the rank of the Jacobian matrix is below $k$. 

Similarly to the polynomial case, we then have to directly rely on Proposition~\ref{pro:general-uniqueness} and search for an analytic function $H$ such that $H(g)=0$. Thus, we define the $features$ method to return a matrix consisting of features expected in the PDE, e.g., $u_{\alpha^i}$, $u_{\alpha^i}^2$, $\exp(u_{\alpha^i})$ to build the feature matrix $U$. Using this as an input to S-FRanCo shows whether there exists a unique PDE that can be described as a linear combination of the chosen features.

\subsection{Continuous and Smooth PDEs}

Theorem~\ref{thm:uniqueness-c-infty-functions} shows that for uniqueness of the PDE in $C^p$ for any $0\leq p \leq\infty$, we need to check if the image of $g:U\rightarrow\R^k$, with $U\subset\R^{m+1}$ open, is dense in $\R^k$. Given finitely many evaluations of $g$ this is not rigorously possible. However, this is not a severe problem as most PDEs can be described by algebraic and analytic functions.

\section{Numerical Experiments} \label{sec:experiments}

In this section, we consider PDEs and their solution and determine whether there exists another PDE that is solved by the given function by using the theory developed in the last sections. Afterward, we apply the algorithms introduced in Section~\ref{sec:numerical-tests} to determine uniqueness numerically without assuming any prior knowledge. The source code of our experiments is publicly available in the anonymous git repository https://anonymous.4open.science/r/physical-law-learning-uniqueness-98CE.

\subsection{Linear PDEs}

We start again with linear PDEs. In Section~\ref{sec:experiments-linear-non-uniqueness} we consider a function that solves more than one linear PDE. We observe that the naive algorithm FRanCo cannot show the ambiguity of the PDE, while the improved version S-FRanCo does succeed. For this reason, we then solely apply S-FRanCo to the function in Section~\ref{sec:experiments-linear-uniqueness} to prove the uniqueness of a linear PDE.

\subsubsection{Non-Uniqueness} \label{sec:experiments-linear-non-uniqueness}

As a first experiment, we consider the function $u:\R^2\rightarrow\R,u(t,x)=\exp(x-at)$, with $a=3$. This function solves the linear PDE $u_t=-au_x$. The goal is now to determine if this is the unique linear PDE in the form $u_t=F(u,u_x)$ which is solved by $u$. For this, we investigate if the function $g\coloneqq(u,u_x)$ is non-trivially annihilated in the set of linear functions. By Corollary~\ref{cor:uniqueness-for-linear-pdes}, we know that this is equivalent to $u$ and $u_x$ being linearly independent. Since $u=u_x$ holds, the function $g=(u,u_x)$ is non-trivially annihilated in the set of linear functions and we easily see that $u$ also solves $u_t=-au$. 

We now show the non-uniqueness numerically. For this, we sample $u(t,x)$ on the square $[0,10]^2$ with 200 measurements for $t$ and 300 for $x$ without using noise. The derivatives are computed numerically from this data using second-order finite differences.

FRanCo should reveal that multiple linear PDEs $u_t=F(u,u_x)$ exist, which are solved by $u$. However, computing the derivative using the second-order finite differences introduces numerical errors. This causes linear independence among $u$ and $u_x$, as the lowest singular value is 10, see also Section~\ref{sec:numerical-tests-linear-pdes}.

Following the ideas from the last section, we apply S-FRanCo with linear features and compute the derivatives for different orders for finite differences. As described in the last section, \mbox{S-FRanCo} computes the lowest singular values for all orders, and their exponential decay to 0 indicates non-uniqueness. This is expected to happen if the exact matrix does not have full rank, since higher order approximations have higher order residual terms \citep{Press20007}. Indeed Figure~\ref{fig:sv-plot-linear-pde-ambiguous} shows precisely this behaviour. Thus, we can be certain that the matrix $(u(t_i,x_j),u_x(t_i,x_j))_{i,j}\in\R^{60,000\times2}$ is singular and, therefore, $u$ and $u_x$ are linearly dependent. This allows to conclude that there do exist multiple linear PDEs $u_t=F(u,u_x)$ solved by $u$.

\begin{figure}[h]
    \centering
    \includegraphics[width=0.5\textwidth]{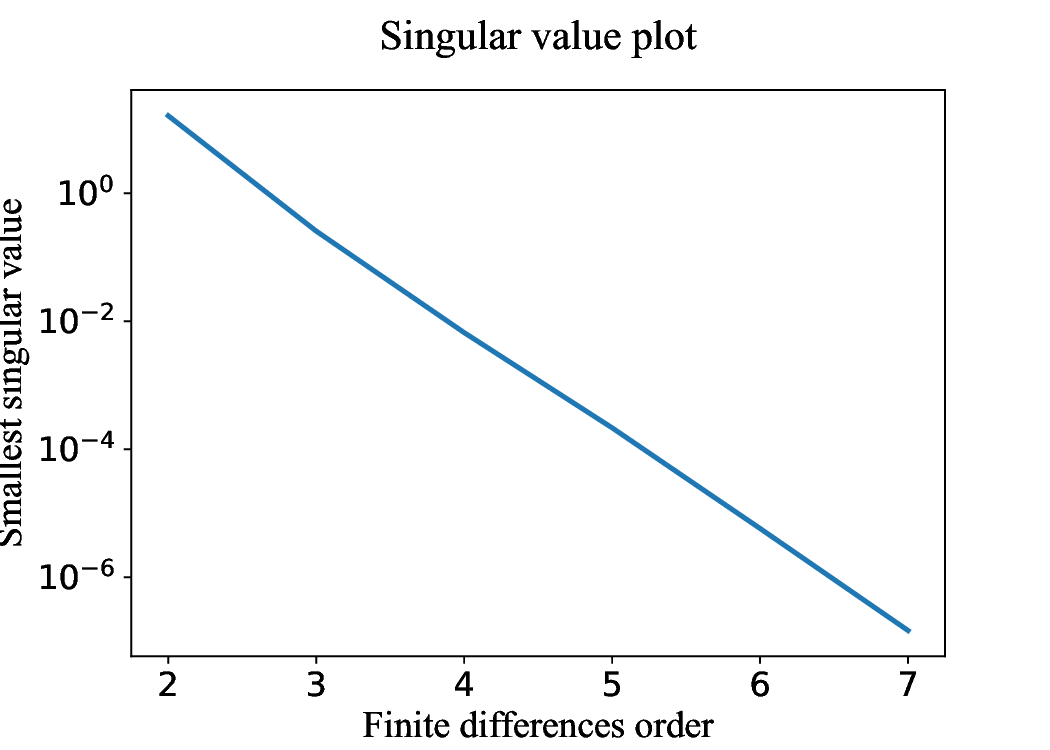}
    \caption{Plot of the lowest singular value of $(u(t_i,x_j),u_x(t_i,x_j))_{i,j}\in\R^{60,000\times2}$, where $u_x$ was computed using finite differences of different orders for $u(t,x)=\exp(x-at)$.}
    \label{fig:sv-plot-linear-pde-ambiguous}
\end{figure}

\subsubsection{Uniqueness} \label{sec:experiments-linear-uniqueness}

 We consider the linear PDE $u_t=au+bu_x$, with $a=1$ and $b=2$, $u(0,x)=x$, $u\in C^\infty(\R^2)$, which is solved by
\begin{equation}
    u(t,x)=(x+bt)\exp(at),\; t,x\in\R.
\end{equation} 
The question we address in this subsection is, whether there exists another linear PDE $u_t=F(u,u_x)$ which is solved by $u$. As the functions $u:\R^2\rightarrow\R, u(t,x)=(x+bt)\exp(at)$ and $u_x:\R^2\rightarrow\R, u_x(t,x)=\exp(at)$ are linearly independent, Corollary~\ref{cor:uniqueness-for-linear-pdes} yields that $u_t=au+bu_x$ is the unique linear PDE of the form $u_t=F(u,u_x)$ solved by $u$. 

As before, we now sample $u(t,x)$ on the square $[0,10]^2$ with 200 measurements for $t$ and 300 for $x$ without using noise. We then numerically assess uniqueness, i.e., the linear dependence of $u$ and $u_x$ by using the plot of the singular values in Figure~\ref{fig:non-uniqueness-for-linear-pdes} computed by S-FRanCo with linear features. Since Figure~\ref{fig:non-uniqueness-for-linear-pdes} indicates no exponential convergence to 0 of the lowest singular value for increasing accuracy, we showed numerically that $u$ and $u_x$ are linearly independent.

\begin{figure}[h]
    \centering
    \includegraphics[width=0.5\textwidth]{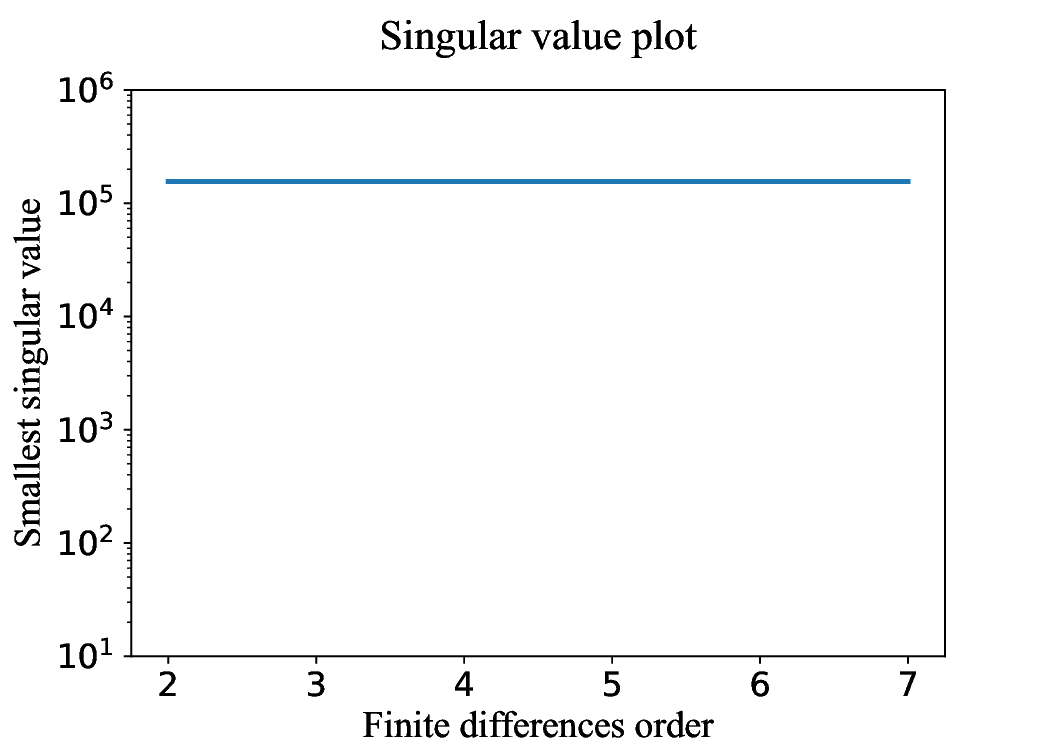}
    \caption{Plot of the lowest singular value of $(u(t_i,x_j),u_x(t_i,x_j))_{i,j}\in\R^{60,000\times2}$, where $u_x$ was computed using finite differences of different orders for $u(t,x)=(x+bt)\exp(at)$. (adapted from \citet{scholl2023icassp})}
    \label{fig:non-uniqueness-for-linear-pdes}
\end{figure}

Note that it was important to specify the derivatives used, i.e., that we addressed uniqueness for PDEs of the form $u_t=F(u,u_x)$ for $F:\R^2\rightarrow\R$ linear. To illustrate this we check now analytically if $u_t=au+bu_x$ is the unique PDE of the form $u_t=G(u,u_x,u_{xx})$ solved by $u$, for $G:\R^3\rightarrow\R$ linear. However, since $u_{xx}=0$ holds, the functions $u$, $u_x$ and $u_{xx}$ are linear dependent. This implies that $u_t=au+bu_x$ is not unique anymore when we allow $u_{xx}$ to be used. An additional PDE solved by $u$ is
\begin{equation}
    u_t=au+bu_x+ u_{xx}.
\end{equation} 
Thus, it becomes apparent that infinitely many linear PDEs exist solved by $u$, but only one of the forms $F(u,u_x)=u_t$ exists.

\subsection{Polynomial and Algebraic PDEs}

For the experiments with polynomial PDEs, we start by revisiting the Korteweg–De Vries equation in Section~\ref{sec:korteweg-de-vries-equation} and consider afterwards a PDE with an algebraic solution in Section~\ref{sec:experiments-polynomials-algebraic-function}.

\subsubsection{Korteweg–De Vries Equation} \label{sec:korteweg-de-vries-equation}

In this section, we consider Example~\ref{ex:kdv-equation} again and aim to verify that our algorithms are able to show the encountered ambiguity displayed there. For our experiments we set $a=0$ and $c=1$ and sample $u(t,x)$ on the square $[0,10]^2$ with 200 measurements for $t$ and 300 for $x$. First we check uniqueness for $u_t=F(u,u_x,u_{xx},u_{xxx})$ for linear functions $F$. For this, we apply S-FRanCo with linear features and visualize the results in Figure~\ref{fig:kdv-linear-pdes}. The clear absence of exponential decay shows that the one-way wave equation is indeed the unique linear PDE. 

\begin{figure}[h]
    \centering
    \includegraphics[width=0.5\textwidth]{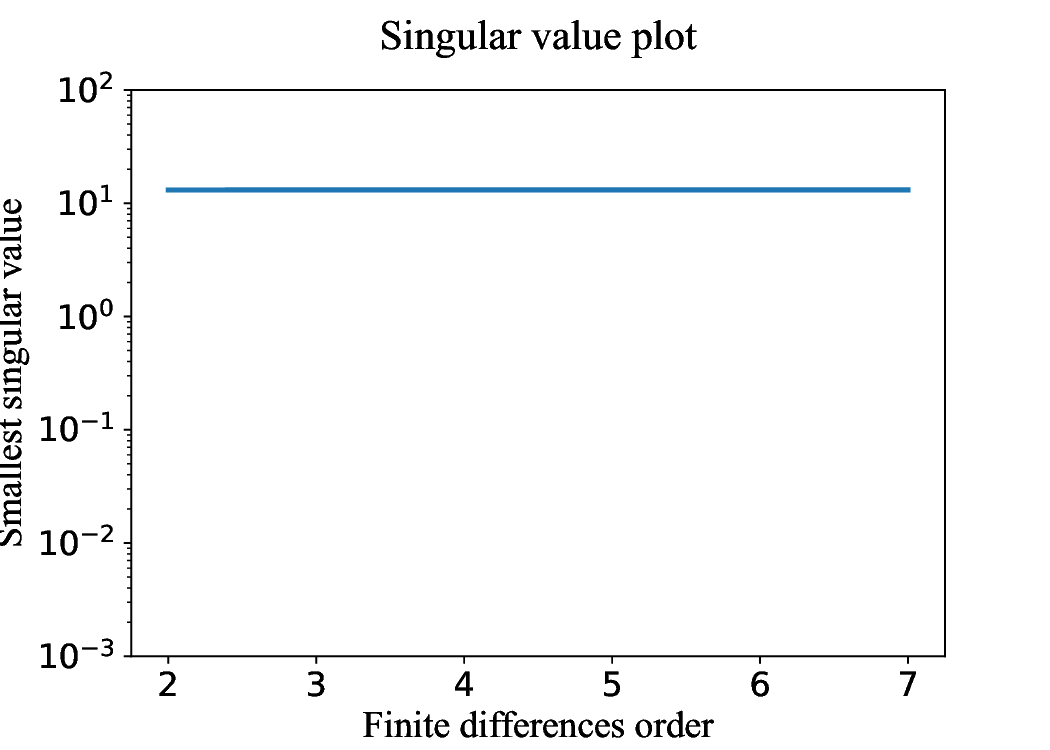}
    \caption{Plot of the lowest singular value of $(u(t_i,x_j),u_x(t_i,x_j),u_{xx}(t_i,x_j),u_{xxx}(t_i,x_j))_{i,j}\in\R^{60,000\times4}$, where the derivatives were computed using finite differences of different orders for $u(t,x)=\frac{c}{2}sech^2(\frac{\sqrt{c}}{2}(x-ct-a))$.}
    \label{fig:kdv-linear-pdes}
\end{figure}

Next, we aim to show numerically that the one-way wave equation is not the unique polynomial PDE. As $u$ is not algebraic, we cannot use Theorem~\ref{thm:algebraic-dependenceof-m+1-polynomials} and~\ref{thm:Jacobian-criterion-for-algebraic-independence}. Following the plan from Section~\ref{sec:numerical-experiment-polynomials}, we would usually start by applying JRC, i.e., by trying to apply the Jacobian criterion for analytic functions. However, there are more functions than variables involved and, therefore, the Jacobian criterion is not helpful as the rank of the Jacobian is at most $m+1=2<4=k$.

Hence, we directly jump to the second step, see Section~\ref{sec:numerical-experiment-polynomials}, namely using S-FRanCo with monomial features. This means we construct a library of monomials of the derivatives and check them for linear independence. Naturally, we start with monomials up to order 2, i.e., we construct the feature matrix

\begin{equation}
U = \left(\begin{array}{cccccccc}
| & | & | & | & | & | & \hdots & |\\
u & u_x & u_{xx} & u_{xxx} & u^2 & uu_x & \hdots & u_{xxx}^2\\
| & | & | & | & | & | & \hdots & |
\end{array}\right).    
\end{equation}

The rank computation of $U$ can be done using S-FRanCo and is shown in Figure~\ref{fig:kdv-polynomial-pdes}. We see that $U$ is indeed not full rank, which means that there exists a polynomial $p$ of degree 2 such that $p(u,u_x,u_{xx},u_{xxx})=0$ and the one-way wave equation is, thus, not the unique polynomial PDE solved by $u$.

\begin{figure}[h]
    \centering
    \includegraphics[width=0.5\textwidth]{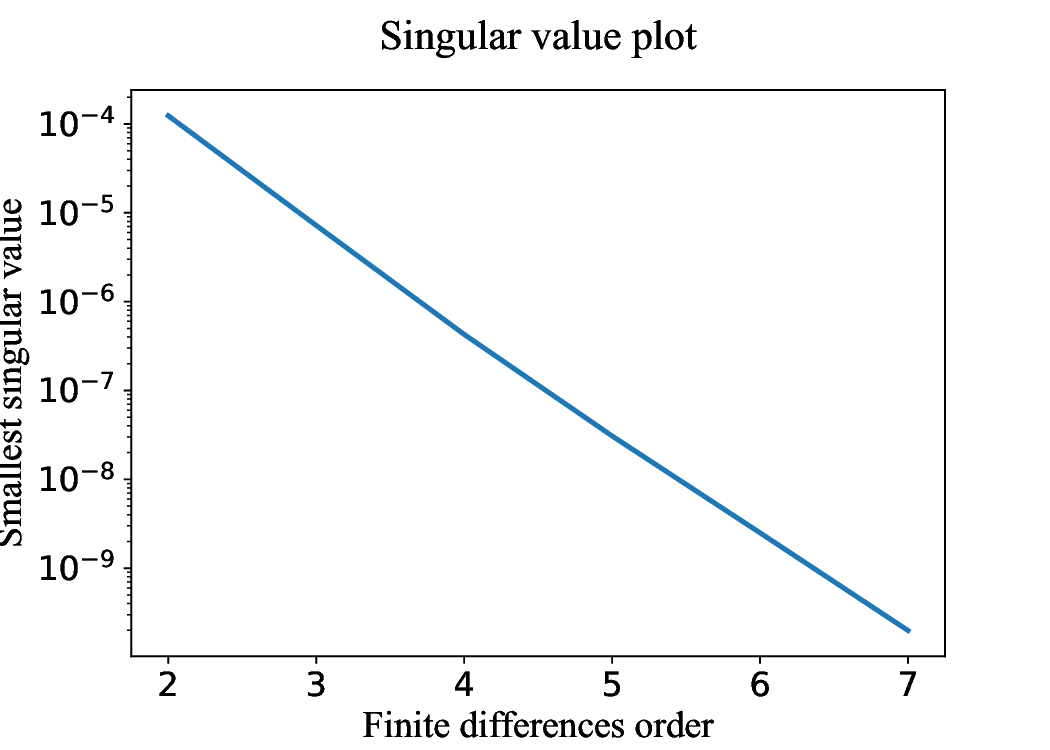}
    \caption{Plot of the lowest singular value of the feature matrix consisting out of all monomials up to degree 2 of $u$, $u_x$, $u_{xx}$ and $u_{xxx}$, where the derivatives were computed using finite differences of different orders for $u(t,x)=\frac{c}{2}sech^2(\frac{\sqrt{c}}{2}(x-ct-a))$. (adapted from \citet{scholl2023icassp})}
    \label{fig:kdv-polynomial-pdes}
\end{figure}

\subsubsection{Algebraic Solution} \label{sec:experiments-polynomials-algebraic-function}

In this section we consider the algebraic function $u:\R_{>0}^2\rightarrow\R,u(t,x)=1/(t+x)$ which solves the linear PDE $u_t=u_x$ and the polynomial PDE $u_t=-u^2$. For our experiments, we sample $u(t,x)$ on the square $[1,5]^2$ with 200 measurements for $t$ and 300 for $x$. We assume that we know that $u$ is an algebraic function. We then aim to show the non-uniqueness of polynomial PDEs of the form $u_t=F(u,u_x)$ by applying the Jacobian criterion.  Figure~\ref{fig:Jacobian-polynomial-ambiguous-algebraic-function} shows the smallest singular value of the Jacobian at different data points $(t_i,x_j)$, as computed by JRC. The upper image was created by computing the derivatives using $2^{nd}$ order finite differences, and the lower image using $7^{th}$ order finite differences. We observe a clear trend from singular values around $10^{-5}$ to $10^{-12}$ and, thus, deduce that the Jacobian is at no point $(t_i,x_j)$ full rank. As $u$ is algebraic, Theorem~\ref{thm:Jacobian-criterion-for-algebraic-independence} yields that $u$ solves multiple polynomial PDEs.

\begin{figure}[h]
    \centering
    \includegraphics[width=0.5\textwidth]{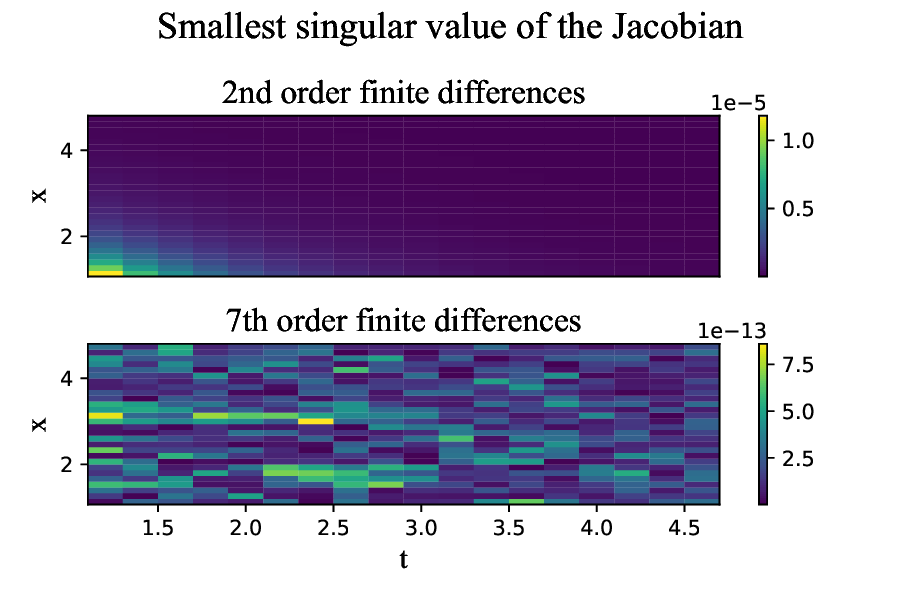}
    \caption{Smallest singular value of the Jacobian at different points $(t_i,x_j)$ of $g=(u,u_x)$ for $u(t,x)=1/(x+t)$. For the upper image, the derivatives were computed using 2nd-order finite differences and, for the lower image, 7th-order finite differences were used. (adapted from \citet{scholl2023icassp})}
    \label{fig:Jacobian-polynomial-ambiguous-algebraic-function}
\end{figure}

\subsection{Analytic PDEs}

In this last section, we investigate the usefulness of the Jacobian criterion for analytic PDEs. For this, we first consider a case where the function solves multiple analytic PDEs and reveals that its Jacobian is indeed never full rank. The second case then considers functions that solve only one analytic PDE. We then prove uniqueness numerically by showing that the Jacobian has full rank at all points.

\subsubsection{Non-Uniqueness}

Consider the PDE 
\begin{equation}
    u_t=u_x, \; u(0,x)=sin(x)
\end{equation}
which is solved by $u:\R^2\rightarrow\R,u(t,x)=sin(x+t)$. We sample $u(t,x)$ on the square $[0,5]^2$ with 200 equispaced measurements for $t$ and 300 for $x$. The goal is to check if $u_t=u_x$ is the only analytic PDE of the form $u_t=F(u,u_x)$, which is solved by $u$. Thus, we investigate the function
\begin{equation}
    g(t,x)\coloneqq(u(t,x),u_x(t,x))=(sin(t+x),cos(t+x)).
\end{equation} 
Obviously, $u$ and $u_x$ are linearly independent and, therefore, $u_t=u_x$ is the unique linear PDE only using $u$ and $u_x$ which is solved by $u$.

Since $\mathcal{D}=\{(x,y)\in\R^2:x^2+y^2=1\}$ we obtain $\lambda^2(\mathcal{D})=0$. Thus, the Jacobian has nowhere full rank, so Proposition~\ref{thm:Jacobian-condition-for-uniqueness-of-analytic-functions} and \ref{pro:measure-condition-for-uniqueness-of-analytic-functions} are not applicable. We now intend to verify this also numerically. Figure~\ref{fig:Jacobian-analytic-ambiguous} shows the smallest singular value of the Jacobian at different data points $(t_i,x_j)$ as computed by JRC. We clearly see the trend to a small singular value in every data point, indicating that the Jacobian matrix has never full rank. Therefore, we cannot deduce that the PDE was the unique analytic PDE. 

\begin{figure}[h]
    \centering
    \includegraphics[width=0.5\textwidth]{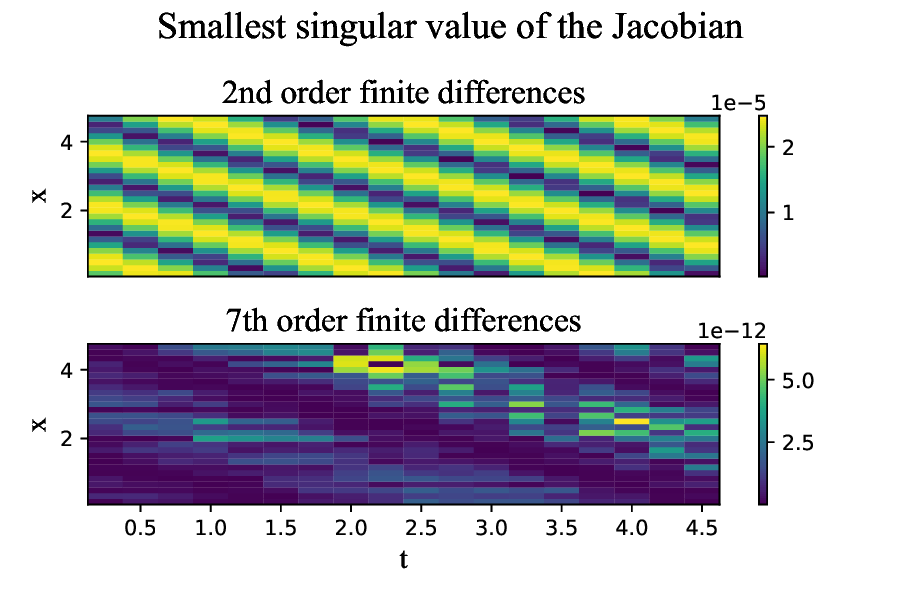}
    \caption{Smallest singular value of the Jacobian at different points $(t_i,x_j)$ of $g=(u,u_x)$ for $u(t,x)=sin(x+t)$. For the upper image, the derivatives were computed using $2^{nd}$-order finite differences and, for the lower image, 7th-order finite differences were used.}
    \label{fig:Jacobian-analytic-ambiguous}
\end{figure}

Thus, we continue with S-FRanCo and check the singular value plot of a feature library using monomials. This is displayed in Figure~\ref{fig:sv-plot-analytic-pdes-ambiguous} for monomials up to degree 2, which shows that $u_t=u_x$ is not the unique polynomial PDE solved by $u$. Indeed one can verify that the function $u$ also solves $u_t=u_x+u^2+u_x^2-1$.

\begin{figure}[h]
    \centering
    \includegraphics[width=0.45\textwidth]{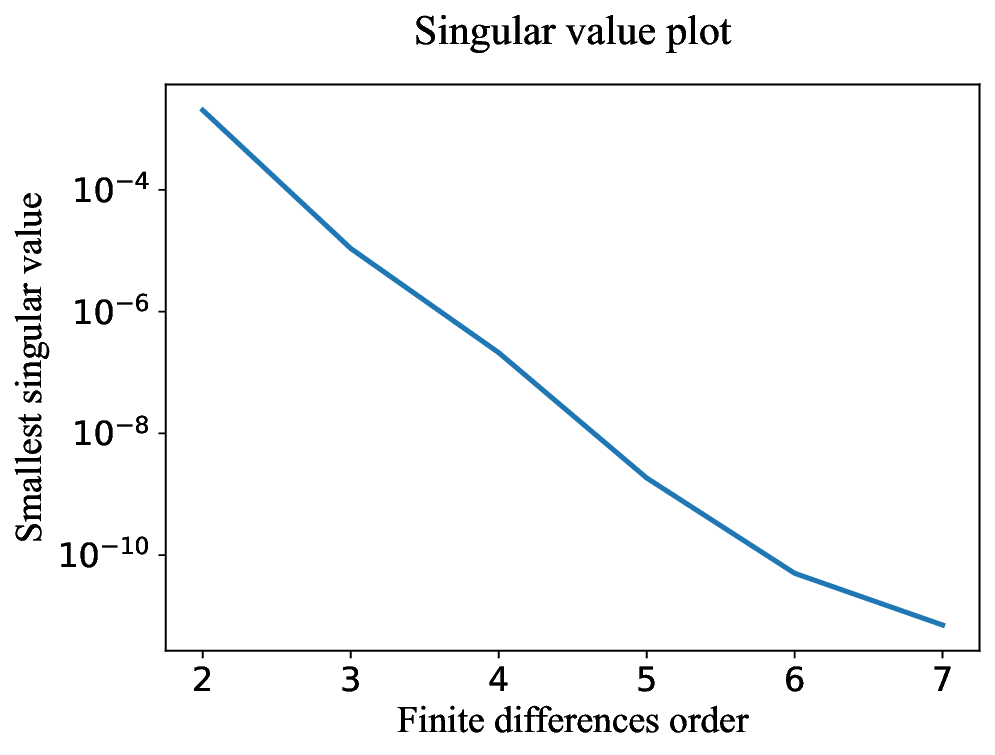}
    \caption{Plot of the lowest singular value of the feature matrix consisting out of all monomials up to degree 2 of $u$ and $u_x$, where the derivatives were computed using finite differences of different orders for $u(t,x)=sin(x+t)$.}
    \label{fig:sv-plot-analytic-pdes-ambiguous}
\end{figure}

\subsubsection{Uniqueness}

We now continue the investigation of the function $u:\R^2\rightarrow\R, u(t,x)=(x+bt)\exp(at)$, with $a=2$ and $b=3$, which we know from Section~\ref{sec:experiments-linear-uniqueness} solves only the linear PDE $u_t=au+bu_x$. We again sample $u(t,x)$ on the square $[0,10]^2$ with 200 equispaced measurements for $t$ and 300 for $x$. This time we focus on the question, of whether it is also the unique analytic PDE solved by $u$. Therefore, we apply JRC to check if there exists at least one data point for which the Jacobian matrix has full rank. Figure~\ref{fig:Jacobian-analytic-unique} shows that the Jacobian has full rank at every point and, thus, $u_t=au+bu_x$ is the unique analytic PDE solved by $u$. This can also be seen theoretically since the image of $(u,u_x)$ is $\R\times\R_{>0}$ and has therefore a non-zero measure.

\begin{figure}[h]
    \centering
    \includegraphics[width=0.5\textwidth]{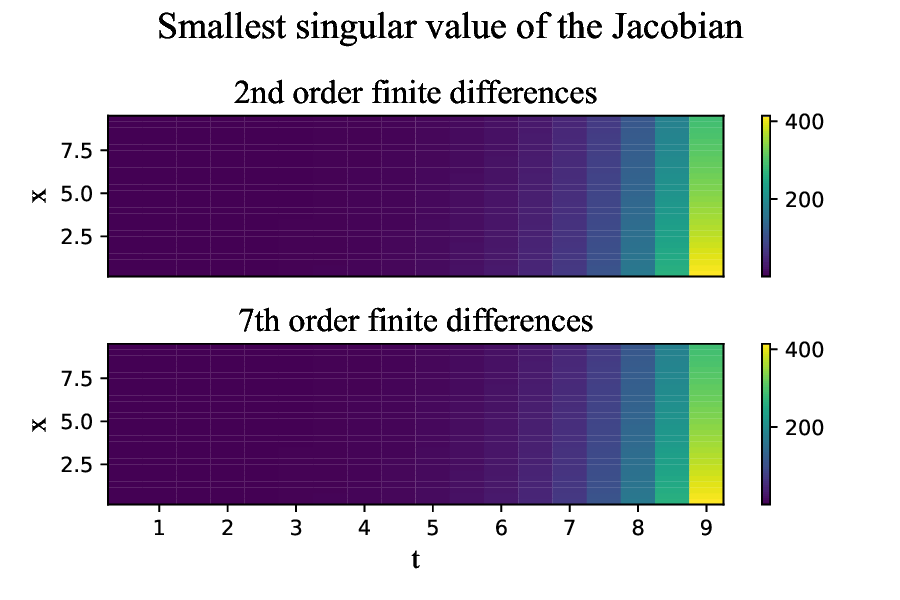}
    \caption{Smallest singular value of the Jacobian at different points $(t_i,x_j)$ of $g=(u,u_x)$ for  $u(t,x)=(x+bt)\exp(at)$, with $a=1$ and $b=2$. For the upper image, the derivatives were computed using 2nd-order finite differences and, for the lower image, 7th-order finite differences were used.}
    \label{fig:Jacobian-analytic-unique}
\end{figure}

At last, we consider a PDE which is non-algebraic but analytic:
\begin{equation} \label{eq:non-algebraic-pde}
    u_t=u_x-\frac{u}{u_x}\sin(u_x).
\end{equation}
We first show, that \eqref{eq:non-algebraic-pde} is solved by $u:\R_{>0}\times\R\rightarrow\R, u(t,x)=(x+t)v(t)$, for $t>0$, where $v:\R_{>0}\rightarrow\R, v(t)=\arcsin(\sech(t))$. We start with computing
\begin{equation}
    v_t(t)=-\frac{\tanh(t)\sech(t)}{\sqrt{1-\sech^2(t)}}=-\sech(t)=-sin(v(t)),
\end{equation}
for $t>0$ and, as $u_x=v$, we obtain for all $t>0$ and $x\in\R$
\begin{equation}
    u_t=v(t)-(x+t)v_t(t)=u_x-\frac{u}{u_x}sin(u_x).
\end{equation}

This is well-defined, since $u_x(t,x)\neq0$ holds for $t>0$. This follows from the fact that $\sech(t)=\frac{2e^t}{e^{2t}+1}\in(0,1)$, for $t>0$, and $\arcsin((0,1))=(0,\pi/2)$. We sample $u(t,x)$ on the square $[1,5]^2$ with 200 equispaced measurements for $t$ and 300 for $x$. We now ask whether \eqref{eq:non-algebraic-pde} is the unique analytic PDE of the form $u_t=F(u,u_x)$ which is solved by $u$. For this, we apply the Jacobian criterion using JRC and check the singular values in Figure~\ref{fig:Jacobian-pure-analytic-unique}. Again, we see no trend towards 0 for the least singular values and, therefore, we deduce that \eqref{eq:non-algebraic-pde} is the unique analytic PDE solved by $u$.

\begin{figure}[h]
    \centering
    \includegraphics[width=0.5\textwidth]{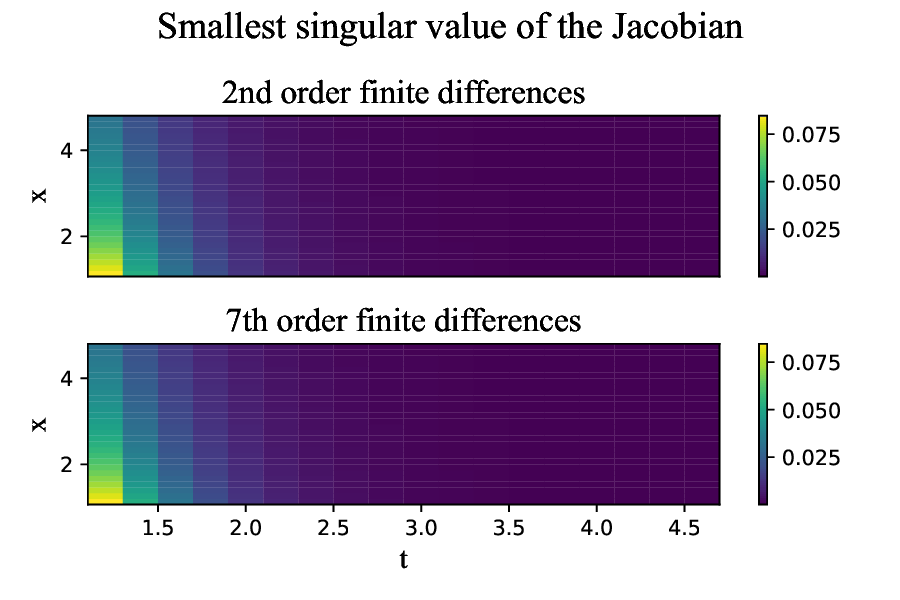}
    \caption{Smallest singular value of the Jacobian at different points $(t_i,x_j)$ of $u(t,x)=(x+t)\arccos(\sech(-t))$. For the upper image, the derivatives were computed using 2nd-order finite differences and, for the lower image, 7th-order finite differences were used. (adapted from \citet{scholl2023icassp})    }
    \label{fig:Jacobian-pure-analytic-unique}
\end{figure}

This follows also theoretically by Proposition~\ref{pro:measure-condition-for-uniqueness-of-analytic-functions}, as the image of $(u,u_x)$ is $\R\times v(\R)$ which has non-zero measure.

\section{Conclusion}

Although the non-uniqueness of governing equations is a well-addressed problem in the classical parameter estimation setting and it is even more important for modern approaches for the symbolic recovery of differential equations which search an even bigger space of equations, this problem has\textemdash{}to the best of our knowledge\textemdash{}never been addressed in this context in the literature apart from \citet{Rudy2017DatadrivenDO}, see Example~\ref{ex:kdv-equation}. Furthermore, \citet{Rudy2017DatadrivenDO} were only able to resolve the non-uniqueness problem for one specific equation because they knew in advance which equation they wanted to learn. This motivates the theoretical investigation of the question: When does a function $u$ solve a unique PDE of the form $F(u_{\alpha^1},...,u_{\alpha^k})=\frac{\partial^n u}{\partial^n t}$, where $F:\R^k\rightarrow\R$ belongs to a specific function class?

We address this question by proving necessary and sufficient conditions for various classes of differential equations, which foster our understanding of the uniqueness problem and can also be assessed numerically to bring immediate value to practitioners. These results also reveal how seldom one can expect to achieve uniqueness, showing the importance of a thorough analysis. 

In future work, we aim to derive results for specific cases common in the symbolic recovery literature to derive stronger uniqueness statements, e.g., we aim to make use of complexity restrictions common in symbolic regression. Additionally, we are interested in restricting the classes to physically more meaningful equations by exploiting symmetries. Furthermore, a current limitation of the proposed algorithms is their vulnerability to noise, which will be an important topic for future work.

\section*{Acknowledgments}

This work of P. Scholl, G. Kutyniok, and H. Boche was supported in part by the ONE Munich Strategy Forum (LMU Munich, TU Munich, and the Bavarian Ministery for Science and Art).

G. Kutyniok acknowledges support from the Konrad Zuse School of Excellence in Reliable AI (DAAD), the Munich Center for Machine Learning (BMBF) as well as the German Research Foundation under Grants DFG-SPP-2298, KU 1446/31-1 and KU 1446/32-1 and under Grant DFG-SFB/TR 109, Project C09 and the German Federal Ministry of Education and Research under Grant MaGriDo.

This work of H. Boche was supported in part by the German Federal Ministry of Education
and Research (BMBF) under Grant 16ME0442.

\bibliographystyle{plainnat}  
\bibliography{references}  

@PREAMBLE{
 "\providecommand{\noopsort}[1]{}" 
 # "\providecommand{\singleletter}[1]{#1}%" 
}

@book{Krantz1992APO,
  title={A primer of real analytic functions},
  author={Krantz, Steven G and Parks, Harold R},
  year={2002},
  publisher={Springer Science \& Business Media},
address={Germany}
}

@article{Neelon,
 ISSN = {00029939, 10886826},
 URL = {http://www.jstor.org/stable/2162018},
 author = {Tejinder S. Neelon},
 journal = {Proceedings of the American Mathematical Society},
 number = {9},
 pages = {2531--2535},
 publisher = {American Mathematical Society},
 title = {On Solutions of Real Analytic Equations},
 urldate = {2022-07-12},
 volume = {125},
 year = {1997}
}

@article{Mityagin,
author = {Mityagin, Boris},
year = {2015},
month = {12},
pages = {},
title = {The Zero Set of a Real Analytic Function},
volume = {107},
journal = {Mathematical Notes},
doi = {10.1134/S0001434620030189}
}

@article{Rudy2017DatadrivenDO,
  title={Data-driven discovery of partial differential equations},
  author={Samuel H. Rudy and Steven L. Brunton and Joshua L. Proctor and J. Nathan Kutz},
  journal={Science Advances},
  year={2017},
  volume={3}
}

@book{Press20007,
author = {Press, William H. and Teukolsky, Saul A. and Vetterling, William T. and Flannery, Brian P.},
title = {Numerical Recipes 3rd Edition: The Art of Scientific Computing},
year = {2007},
isbn = {0521880688},
publisher = {Cambridge University Press},
address = {USA},
edition = {3}
}

@article{Kasparian2010,
author = {Azniv Kasparian},
title = {{Lectures on Curves, Surfaces and Projective Varieties (A Classical View of Algebraic Geometry) by Mauro C. Beltrametti, Ettore Carletti, Dionisio Gallarati and Giacomo M. Bragadin}},
volume = {18},
journal = {Journal of Geometry and Symmetry in Physics},
number = {none},
publisher = {Bulgarian Academy of Sciences, Institute of Mechanics},
pages = {87 -- 92},
year = {2010},
doi = {jgsp/1495677896},
URL = {https://doi.org/}
}

@article{EhrenborgRota1993,
author = {Ehrenborg, Richard and Rota, Gian-Carlo},
title = {Apolarity and Canonical Forms for Homogeneous Polynomials},
year = {1993},
issue_date = {May 1993},
publisher = {Academic Press Ltd.},
address = {GBR},
volume = {14},
number = {3},
issn = {0195-6698},
url = {https://doi.org/10.1006/eujc.1993.1022},
doi = {10.1006/eujc.1993.1022},
journal = {Eur. J. Comb.},
month = {may},
pages = {157–181},
numpages = {25}
}

@book{Morandi1996,
author = {Morandi, Patrick},
title = {Field and Galois Theory},
year = {1996},
isbn = {978-0-387-94753-2},
publisher = {Springer New York, NY},
address = {USA},
edition = {1}
}

@book{Chirka1989,
author = {Chirka, E.M.},
title = {Complex Analytic Sets},
year = {1989},
isbn = {978-94-010-7565-4},
publisher = {Springer Netherland},
address = {Netherland},
edition = {1}
}

@article{Acquistapace2017,
title = {Some results on global real analytic geometry},
journal = {Contemporary Mathematics},
volume = {697},
year = {2017},
doi = {https://doi.org/10.1090/conm/697},
url = {http://www.ams.org/books/conm/697/},
author = {Acquistapace, F. and Broglia, F. and Fernando, J. F.},
}

@book{Burden2015,
author = {Burden, Richard L. and Burden, Annette M.},
title = {Numerical Analysis},
year = {2015},
isbn = {978-0-538-73351-9},
publisher = {Youngstown State University},
address = {Youngstown, Ohio, USA},
edition = {10}
}

@article{Abhyankar1976,
 ISSN = {00029947},
 URL = {http://www.jstor.org/stable/1997583},
 author = {S. S. Abhyankar and T. T. Moh},
 journal = {Transactions of the American Mathematical Society},
 pages = {77--87},
 publisher = {American Mathematical Society},
 title = {On Analytic Independence},
 urldate = {2022-09-13},
 volume = {219},
 year = {1976}
}

@book{thompson2013ordinary,
  title={Ordinary Differential Equations},
  author={Walter, W.},
  isbn={9781461206019},
  lccn={98004754},
  series={Graduate Texts in Mathematics},
  address={USA},  
year={2013},
  publisher={Springer New York}
}

@article{bongard2007lipson,
author = {Josh Bongard  and Hod Lipson },
title = {Automated reverse engineering of nonlinear dynamical systems},
journal = {Proceedings of the National Academy of Sciences},
volume = {104},
number = {24},
pages = {9943-9948},
year = {2007},
doi = {10.1073/pnas.0609476104}
}

@article{schmidt2009lipson,
author = {Michael Schmidt  and Hod Lipson },
title = {Distilling Free-Form Natural Laws from Experimental Data},
journal = {Science},
volume = {324},
number = {5923},
pages = {81-85},
year = {2009},
doi = {10.1126/science.1165893},
URL = {https://www.science.org/doi/abs/10.1126/science.1165893},
eprint = {https://www.science.org/doi/pdf/10.1126/science.1165893},
}

@article{xu2020,
title = {DLGA-PDE: Discovery of PDEs with incomplete candidate library via combination of deep learning and genetic algorithm},
journal = {Journal of Computational Physics},
volume = {418},
pages = {109584},
year = {2020},
issn = {0021-9991},
doi = {https://doi.org/10.1016/j.jcp.2020.109584},
url = {https://www.sciencedirect.com/science/article/pii/S0021999120303582},
author = {Hao Xu and Haibin Chang and Dongxiao Zhang}
}

@article{acar1993identification,
  title={Identification of the coefficient in elliptic equations},
  author={Acar, Robert},
  journal={SIAM journal on control and optimization},
  volume={31},
  number={5},
  pages={1221--1244},
  year={1993},
  publisher={SIAM}
}

@article{knowles2001parameter,
  title={Parameter identification for elliptic problems},
  author={Knowles, Ian},
  journal={Journal of computational and applied mathematics},
  volume={131},
  number={1-2},
  pages={175--194},
  year={2001},
  publisher={Elsevier}
}

@article{hong2020global,
  title={Global identifiability of differential models},
  author={Hong, Hoon and Ovchinnikov, Alexey and Pogudin, Gleb and Yap, Chee},
  journal={Communications on Pure and Applied Mathematics},
  volume={73},
  number={9},
  pages={1831--1879},
  year={2020},
  publisher={Wiley Online Library}
}

@article{ovchinnikov2021computing,
  title={Computing all identifiable functions of parameters for ODE models},
  author={Ovchinnikov, Alexey and Pillay, Anand and Pogudin, Gleb and Scanlon, Thomas},
  journal={Systems \& Control Letters},
  volume={157},
  pages={105030},
  year={2021},
  publisher={Elsevier}
}

@article{pohjanpalo1978system,
  title={System identifiability based on the power series expansion of the solution},
  author={Pohjanpalo, Hannu},
  journal={Mathematical biosciences},
  volume={41},
  number={1-2},
  pages={21--33},
  year={1978},
  publisher={Elsevier}
}

@article{VAJDA1984,
title = {Structural Identifiability of Linear, Bilinear, Polynomial and Rational Systems},
journal = {IFAC Proceedings Volumes},
volume = {17},
number = {2},
pages = {717-722},
year = {1984},
note = {9th IFAC World Congress: A Bridge Between Control Science and Technology, Budapest, Hungary, 2-6 July 1984},
issn = {1474-6670},
doi = {https://doi.org/10.1016/S1474-6670(17)61056-5},
url = {https://www.sciencedirect.com/science/article/pii/S1474667017610565},
author = {Vajda, s.}
}

@inproceedings{alberti2019calderon,
  title={Calder{\'o}n’s inverse problem with a finite number of measurements},
  author={Alberti, Giovanni S and Santacesaria, Matteo},
  booktitle={Forum of mathematics, sigma},
  volume={7},
  pages={e35},
  year={2019},
  organization={Cambridge University Press}
}

@article{kohn1985determining,
  title={Determining conductivity by boundary measurements II. Interior results},
  author={Kohn, Robert V and Vogelius, Michael},
  journal={Communications on Pure and Applied Mathematics},
  volume={38},
  number={5},
  pages={643--667},
  year={1985},
  publisher={Wiley Online Library}
}

@article{calderon2006inverse,
  title={On an inverse boundary value problem},
  author={Calder{\'o}n, Alberto P},
  journal={Computational \& Applied Mathematics},
  volume={25},
  pages={133--138},
  year={2006},
  publisher={SciELO Brasil}
}

@INPROCEEDINGS{scholl2023icassp,
  author={Scholl, Philipp and Bacho, Aras and Boche, Holger and Kutyniok, Gitta},
  booktitle={ICASSP 2023 - 2023 IEEE International Conference on Acoustics, Speech and Signal Processing (ICASSP)}, 
  title={The Uniqueness Problem of Physical Law Learning}, 
  year={2023},
  volume={},
  number={},
  pages={1-5},
  doi={10.1109/ICASSP49357.2023.10095017}}

@article{jain2019priori,
  title={A priori parameter identifiability in models with non-rational functions},
  author={Jain, Rishabh and Narasimhan, Sridharakumar and Bhatt, Nirav P},
  journal={Automatica},
  volume={109},
  pages={108513},
  year={2019},
  publisher={Elsevier}
}

@article{meshkat2015identifiability,
  title={Identifiability results for several classes of linear compartment models},
  author={Meshkat, Nicolette and Sullivant, Seth and Eisenberg, Marisa},
  journal={Bulletin of Mathematical Biology},
  volume={77},
  pages={1620--1651},
  year={2015},
  publisher={Springer}
}

@article{cobelli1980parameter,
  title={Parameter and structural identifiability concepts and ambiguities: a critical review and analysis},
  author={Cobelli, Claudio and Distefano 3rd, Joseph J},
  journal={American Journal of Physiology-Regulatory, Integrative and Comparative Physiology},
  volume={239},
  number={1},
  pages={R7--R24},
  year={1980},
  publisher={American Physiological Society Bethesda, MD}
}

@article{walter1982global,
  title={Global approaches to identifiability testing for linear and nonlinear state space models},
  author={Walter, Eric and Lecourtier, Yves},
  journal={Mathematics and Computers in Simulation},
  volume={24},
  number={6},
  pages={472--482},
  year={1982},
  publisher={Elsevier}
}

@article{alessandrini1986identification,
  title={An identification problem for an elliptic equation in two variables},
  author={Alessandrini, Giovanni},
  journal={Annali di matematica pura ed applicata},
  volume={145},
  number={1},
  pages={265--295},
  year={1986},
  publisher={Springer}
}

@article{udrescu2020tegmark,
author = {Silviu-Marian Udrescu  and Max Tegmark },
title = {AI Feynman: A physics-inspired method for symbolic regression},
journal = {Science Advances},
volume = {6},
number = {16},
pages = {eaay2631},
year = {2020},
doi = {10.1126/sciadv.aay2631},
URL = {https://www.science.org/doi/abs/10.1126/sciadv.aay2631}}

@article{champion2019,
author = {Kathleen Champion  and Bethany Lusch  and J. Nathan Kutz  and Steven L. Brunton },
title = {Data-driven discovery of coordinates and governing equations},
journal = {Proceedings of the National Academy of Sciences},
volume = {116},
number = {45},
pages = {22445-22451},
year = {2019},
doi = {10.1073/pnas.1906995116},
URL = {https://www.pnas.org/doi/abs/10.1073/pnas.1906995116},
eprint = {https://www.pnas.org/doi/pdf/10.1073/pnas.1906995116}}

@inproceedings{
qian2022dcode,
title={D-{CODE}: Discovering Closed-form {ODE}s from Observed Trajectories},
author={Zhaozhi Qian and Krzysztof Kacprzyk and Mihaela van der Schaar},
booktitle={International Conference on Learning Representations},
year={2022}}

@inproceedings{
kacprzyk2023dcipher,
title={D-{CIPHER}: Discovery of Closed-form Partial Differential Equations},
author={Krzysztof Kacprzyk and Zhaozhi Qian and Mihaela van der Schaar},
booktitle={Thirty-seventh Conference on Neural Information Processing Systems},
year={2023},
url={https://openreview.net/forum?id=jnCPN1vpSR}
}

@article{stephany2022pde,
  title={PDE-LEARN: Using deep learning to discover partial differential equations from noisy, limited data},
  author={Stephany, Robert and Earls, Christopher},
  journal={arXiv preprint arXiv:2212.04971},
  year={2022}
}

@article{both2021choudhury,
title = {DeepMoD: Deep learning for model discovery in noisy data},
journal = {Journal of Computational Physics},
volume = {428},
pages = {109985},
year = {2021},
issn = {0021-9991},
doi = {https://doi.org/10.1016/j.jcp.2020.109985},
url = {https://www.sciencedirect.com/science/article/pii/S0021999120307592},
author = {Gert-Jan Both and Subham Choudhury and Pierre Sens and Remy Kusters},
}

@article{quade2018sparse,
  title={Sparse identification of nonlinear dynamics for rapid model recovery},
  author={Quade, Markus and Abel, Markus and Nathan Kutz, J and Brunton, Steven L},
  journal={Chaos: An Interdisciplinary Journal of Nonlinear Science},
  volume={28},
  number={6},
  year={2018},
  publisher={AIP Publishing}
}

@article{scholl2023parfam,
  title={ParFam--(Neural Guided) Symbolic Regression Based on Continuous Global Optimization},
  author={Scholl, Philipp and Bieker, Katharina and Hauger, Hillary and Kutyniok, Gitta},
  journal={arXiv preprint arXiv:2310.05537},
  year={2023}
}

@article{crabbe2020learning,
  title={Learning outside the black-box: The pursuit of interpretable models},
  author={Crabbe, Jonathan and Zhang, Yao and Zame, William and van der Schaar, Mihaela},
  journal={Advances in neural information processing systems},
  volume={33},
  pages={17838--17849},
  year={2020}
}

@inproceedings{alaa2019meijerg,
 author = {Alaa, Ahmed M. and van der Schaar, Mihaela},
 booktitle = {Advances in Neural Information Processing Systems},
 editor = {H. Wallach and H. Larochelle and A. Beygelzimer and F. d\textquotesingle Alch\'{e}-Buc and E. Fox and R. Garnett},
 title = {Demystifying Black-box Models with Symbolic Metamodels},
 volume = {32},
 year = {2019}
}

@article{kaheman2020sindy,
  title={SINDy-PI: a robust algorithm for parallel implicit sparse identification of nonlinear dynamics},
  author={Kaheman, Kadierdan and Kutz, J Nathan and Brunton, Steven L},
  journal={Proceedings of the Royal Society A},
  volume={476},
  number={2242},
  pages={20200279},
  year={2020},
  publisher={The Royal Society Publishing}
}

@article{chen2021physics,
  title={Physics-informed learning of governing equations from scarce data},
  author={Chen, Zhao and Liu, Yang and Sun, Hao},
  journal={Nature communications},
  volume={12},
  number={1},
  pages={6136},
  year={2021},
  publisher={Nature Publishing Group UK London}
}

@article{Hasan2020,
  title={Learning Partial Differential Equations From Data Using Neural Networks},
  author={Ali Hasan and Jo{\~a}o M. Pereira and Robert J. Ravier and Sina Farsiu and Vahid Tarokh},
  journal={ICASSP 2020 - 2020 IEEE International Conference on Acoustics, Speech and Signal Processing (ICASSP)},
  year={2020},
  pages={3962-3966}
}

@article{martius2017lampert,
author = {Martius, Georg and Lampert, Christoph},
year = {2017},
month = {10},
pages = {},
title = {Extrapolation and learning equations},
 journal={5th International Conference on Learning Representations, ICLR 2017 - Workshop Track Proceedings},
}

@Book{ruffini1813,
 author    = { Ruffini, Paolo },
 title     = { Riflessioni intorno alla soluzione delle equazioni algebraiche generali },
 publisher = { Presso La Società },
 year      = { 1813 }, address   = { Modena }}

@inproceedings{abel1813,
 author    = { Abel, Niels H. },
 title     = { Mémoire sur les équations algébriques, ou l'on démontre l'impossibilité de la résolution de l'équation générale du cinquième degré },
 booktitle = { Œuvres Complètes de Niels Henrik Abel},
 publisher = { Grondahl and Son },
 editors = {Sylow, Ludwig and Lie, Sophus},
 year      = { 1881 }, 
 address   = {Norway }}

@book{rudin1987,
author = {Rudin, Walter},
title = {Real and Complex Analysis, 3rd Ed.},
year = {1987},
isbn = {0070542341},
publisher = {McGraw-Hill, Inc.},
address = {USA}
}

@book{lang2012algebra,
  title={Algebra},
  author={Lang, Serge},
  volume={211},
  year={2012},
  publisher={Springer Science \& Business Media},
  address = {Germany},
}

@book{chicone2006,
  title={Ordinary Differential Equations with Applications},
  author={Chicone, Carmen},
  edition={2},
  year={2006},
  address = {USA},
    series={Texts in Applied Mathematics},
  publisher={Springer Science+Business Media}
}

@inproceedings{la2021contemporary,
  author       = {La Cava, William G.  and
                  Patryk Orzechowski and
                  Bogdan Burlacu and
                  Fabr{\'{\i}}cio Olivetti de Fran{\c{c}}a and
                  Marco Virgolin and
                  Ying Jin and
                  Michael Kommenda and
                  Jason H. Moore},
  _editor       = {Joaquin Vanschoren and
                  Sai{-}Kit Yeung},
  title        = {Contemporary Symbolic Regression Methods and their Relative Performance},
  booktitle    = {Proceedings of the Neural Information Processing Systems Track on Datasets and Benchmarks},
  year         = {2021}}

@Inbook{McConaghy2011,
    author="McConaghy, Trent",
    _editor="Riolo, Rick
    and Vladislavleva, Ekaterina
    and Moore, Jason H.",
    title="FFX: Fast, Scalable, Deterministic Symbolic Regression Technology",
    bookTitle="Genetic Programming Theory and Practice IX",
    year="2011",
    publisher="Springer New York",
    address="New York, NY",
    pages="235--260",
    _isbn="978-1-4614-1770-5",
    doi="10.1007/978-1-4614-1770-5_13",
    _url="https://doi.org/10.1007/978-1-4614-1770-5_13"
}

@article{
Brunton2016,
author = {Steven L. Brunton  and Joshua L. Proctor  and J. Nathan Kutz },
title = {Discovering governing equations from data by sparse identification of nonlinear dynamical systems},
journal = {Proceedings of the National Academy of Sciences},
volume = {113},
number = {15},
pages = {3932--3937},
year = {2016},
doi = {10.1073/pnas.1517384113},
_URL = {https://www.pnas.org/doi/abs/10.1073/pnas.1517384113},
_eprint = {https://www.pnas.org/doi/pdf/10.1073/pnas.1517384113}}

@article{sun2022symbolic,
  title={Symbolic physics learner: Discovering governing equations via monte carlo tree search},
  author={Sun, Fangzheng and Liu, Yang and Wang, Jian-Xun and Sun, Hao},
  journal={arXiv preprint arXiv:2205.13134},
  year={2022}
}

@inproceedings{biggio2021neural,
  author       = {Luca Biggio and
                  Tommaso Bendinelli and
                  Alexander Neitz and
                  Aur{\'{e}}lien Lucchi and
                  Giambattista Parascandolo},
  editor       = {Marina Meila and
                  Tong Zhang},
  title        = {Neural Symbolic Regression that scales},
  booktitle    = {Proceedings of the 38th International Conference on Machine Learning,
                  {ICML} 2021, 18-24 July 2021, Virtual Event},
  series       = {Proceedings of Machine Learning Research},
  volume       = {139},
  pages        = {936--945},
  year         = {2021}}

@inproceedings{kamienny2022end,
  author       = {Pierre{-}Alexandre Kamienny and
                  St{\'{e}}phane d'Ascoli and
                  Guillaume Lample and
                  Fran{\c{c}}ois Charton},
  title        = {End-to-end Symbolic Regression with Transformers},
  booktitle    = {Advances in Neural Information Processing Systems},
  year         = {2022}}

@inproceedings{holt2022deep,
  title={Deep Generative Symbolic Regression},
  author={Holt, Samuel and Qian, Zhaozhi and van der Schaar, Mihaela},
  booktitle={The Eleventh International Conference on Learning Representations},
  year={2023}
}

@inproceedings{martius2016extrapolation,
  author       = {Georg Martius and
                  Christoph H. Lampert},
  title        = {Extrapolation and learning equations},
  booktitle    = {5th International Conference on Learning Representations, Workshop Track Proceedings},
  year         = {2017},
}

@inproceedings{sahoo2018learning,
  title={Learning equations for extrapolation and control},
  author={Sahoo, Subham and Lampert, Christoph and Martius, Georg},
  booktitle    = {Proceedings of the 35th International Conference on Machine Learning},
  series       = {Proceedings of Machine Learning Research},
  volume       = {80},
  pages        = {4439--4447},
  year         = {2018}}

@inproceedings{Schmidt2010AgefitnessPO,
  title={Age-fitness pareto optimization},
  author={Michael D. Schmidt and Hod Lipson},
  booktitle={Genetic and Evolutionary Computation Conference, {GECCO} 2010, Proceedings},
  pages        = {543--544},
  publisher    = {{ACM}},
    address={USA},
  year         = {2010},
doi          = {10.1145/1830483.1830584},
}

@article{schmidt2009distilling,
  title={Distilling free-form natural laws from experimental data},
  author={Michael D. Schmidt and Hod Lipson},
  journal={Science},
  volume={324},
  number={5923},
  pages={81--85},
  year={2009},
  _publisher={American Association for the Advancement of Science},
  doi = {10.1126/science.1165893},
}

@article{cranmer2023interpretable,
  title={Interpretable machine learning for science with PySR and SymbolicRegression. jl},
  author={Cranmer, Miles},
  journal={arXiv preprint arXiv:2305.01582},
  year={2023}
}

@inproceedings{augusto2000symbolic,
  title={Symbolic regression via genetic programming},
  author={Augusto, D. Adriano and Barbosa, Helio J.C.},
  booktitle={Proceedings. Vol.1. Sixth Brazilian Symposium on Neural Networks},
  pages={173--178},
  year={2000},
  _organization={IEEE},
  doi={10.1109/SBRN.2000.889734}}

@inproceedings{petersen2019deep,
  author       = {Brenden K. Petersen and
                  Mikel Landajuela and
                  T. Nathan Mundhenk and
                  Cl{\'{a}}udio Prata Santiago and
                  Sookyung Kim and
                  Joanne Taery Kim},
  title        = {Deep symbolic regression: Recovering mathematical expressions from
                  data via risk-seeking policy gradients},
  booktitle    = {9th International Conference on Learning Representations, {ICLR} 2021},
  year         = {2021},
}

@inproceedings{mundhenk2021symbolic,
  booktitle={Advances in Neural Information Processing Systems},
  volume={34},
  pages={24912--24923},
  year={2021},
  author       = {T. Nathan Mundhenk and
                  Mikel Landajuela and
                  Ruben Glatt and
                  Cl{\'{a}}udio P. Santiago and
                  Daniel M. Faissol and
                  Brenden K. Petersen},
  title={Symbolic regression via deep reinforcement learning enhanced genetic programming seeding},
}

\appendix

\section{Fundamentals of Algebra} \label{app:fundamentals-of-algebra}

In this section, we introduce all definitions and results from algebra that are needed to understand the proofs in Section~\ref{sec:uniqueness-algebraic-pdes}. A more thorough treatment can be found in \citet{Morandi1996} and \citet{lang2012algebra}. 

We start with some basic notation. We denote the \emph{ring of polynomials} over $x=(x_1,...,x_m)$ with coefficients in the field $K$ by $K[x]$. A polynomial $p\in K[x]$ is called \emph{irreducible} over $K$ if there exist no non-constant polynomials $p_1,p_2\in K[x]$ such that $p(x)=p_1(x)p_2(x)$. The \emph{rational function field} $K(x)\coloneqq\{p(x)/q(x):p,q\in K[x],\; q\neq0\}$ is the field of rational functions with variables $x$ and coefficients in $K$.

The next definitions we need concern field extensions.

\begin{definition}
    Let $F$ and $K$ be fields with $F\subset K$. Then, we call $K$ a \emph{field extension} of $F$, denoted by $K/F$. We call an element $\alpha\in K$ \emph{algebraic} over $F$ if there exists a nonzero polynomial $P\in F[x]$ with $P(\alpha)=0$. If every element of $K$ is algebraic over $F$, we say that $K$ is algebraic over $F$ and $K/F$ is called an \emph{algebraic extension}.
\end{definition}

We will use this these definitions to define algebraic functions as algebraic elements over $\C(x)$. 

\begin{definition} \label{def:algebraic-closure}
    A field $F$ is called \emph{algebraically closed} if every polynomial $p\in F[x]$ with degree at least one has a root in $F$. An algebraic field extension $F\subset K$ is called the \emph{algebraic closure} of $F$ if it is algebraically closed. We denote $\overline{F}=K$. Now we can define an \emph{algebraic function} $f$ over $\C$ as an element of the algebraic closure of the rational function field $\C(x)$, i.e., $f\in\overline{\C(x)}$.
\end{definition}

A short note on polynomials and algebraic functions: Formally, polynomials and polynomial functions are two different quantities. A polynomial $p\in K[x]$, for $K$ some field, is a purely symbolic element from a polynomial ring which can be considered a function via the \emph{evaluation homorphism}, which maps a polynomial to the corresponding polynomial function $p_K:K\rightarrow K$. This differentiation is especially important if the evaluation homomorphism is not injective. This happens, for example, for $K=\mathbb{F}_2=\{0,1\}$ as $p(x)=x^2+x$ is unequal to the zero polynomial, even though $p_K(\alpha)=0$ for all $\alpha\in\mathbb{F}_2$. For $K=\R$ or $K=\C$, however, the evaluation homomorphism is injective and, therefore, we do not have to differentiate between polynomials and polynomial functions. This way we can connect the two definitions of algebraic functions, Definition~\ref{def:algebraic-function} and Definition~\ref{def:algebraic-closure}.

Many of the theorems on algebraic dependence of algebraic functions are about algebraic dependence over $\C$. However, we are interested in real-valued functions and algebraic dependence over $\R$. We can still use those results due to the next lemma.

\begin{lemma} \label{lem:real-function-algebraic-dependent}
    Let $f_i:\R\rightarrow\R$, $1\leq n$ be real-valued functions. Then the functions $f_i$ are algebraically independent over $\R$ if and only if $f_i$ are algebraically independent over $\C$.
\end{lemma}
\begin{proof}
    Let $p\in\C[x]$ and decompose it as $p=p_1+ip_2$ with $p_1,p_2\in\R[x]$. Then, we observe that $p(f_1,...,f_n)=0$ is equivalent to
    $p_1(f_1,...,f_n)=0=p_2(f_1,...,f_n)$. 
    This shows that $0\neq p\in\C[x]$ with $p(f_1,...,f_n)=0$ exists if and only if there exists $0\neq q\in\R[x]$ with $q(f_1,...,f_n)=0$.
\end{proof}

\section{Analytic Independence} \label{app:analytic-independence}

A logical consideration is to try if Theorem~\ref{thm:algebraic-dependenceof-m+1-polynomials} and~\ref{thm:Jacobian-criterion-for-algebraic-independence} extend to analytic functions and analytic independence. This would be a powerful extension as it might be a plausible assumption in many cases that $u$ is an analytic function. Also, the proof in \citet{EhrenborgRota1993} shows that it suffices to prove that analytic dependence defines a matroid similar to how algebraic dependence defines a matroid. Analytic analogs of Theorem~\ref{thm:algebraic-dependenceof-m+1-polynomials} and \ref{thm:Jacobian-criterion-for-algebraic-independence} would follow similarly as for algebraic functions and algebraic PDEs as shown in Section~\ref{sec:uniqueness-algebraic-pdes} and in \citet{EhrenborgRota1993}. Let us now define analytic dependence \citep{Abhyankar1976}, before we discuss why a proof as for algebraic function and polynomial/algebraic PDEs cannot exist for analytic functions and analytic PDEs.

\begin{definition} \label{def:analytic-dependence}
    We call functions $f_1,...,f_q:\C^p\rightarrow\C$ \emph{analytically dependent over $\C$} if there exists a non-zero analytic function $P\in C^\omega(\C^q)\backslash \{0\}$ such that $P(f_1(x_1,..,x_p),...,f_q(x_1,..,x_p))=0$. If no such $P$ exists, we call $f_1,...,f_q$ \emph{analytically independent over $\C$}.
\end{definition}

Following the proof in \citet{EhrenborgRota1993} the analytic extension of Theorem~\ref{thm:Jacobian-criterion-for-algebraic-independence} would rely on the extension of Theorem~\ref{thm:algebraic-dependenceof-m+1-polynomials}. However, this approach is not feasible as the following example shows the existence of two analytic functions in one variable which are analytically independent.

\begin{example} \label{ex:two-analytic-function-independent}
    Let $a,b\in\R$ be linearly independent over $\Q$. Then there does not exist a non-zero analytic function $F:U\rightarrow\C$, for $U\subset\C^2$ open and $B_1(0)\subset U$, which fulfills $F(e^{iat},e^{ibt})=0$ for all $t\in\R$. Thus, the functions $t\mapsto e^{iat}$ and $t\mapsto e^{ibt}$ are analytically independent.
\end{example}
\begin{proof}
    Let $F:U\rightarrow\C$ be any analytic function on some open set $U\subset\C^2$ with $B_1(0)\subset U$. Furthermore, let $F(x,y)=\sum_{k,l=1}^\infty c_{kl}x^ky^l$ be its Taylor series around $(x,y)=0$ which we know converges on the entire set $U\subset\R^2$. Let $f(t)= e^{iat}$ and $g(t)= e^{ibt}$ for $t\in\R$. Then, $F(f(t),g(t))=\sum_{k,l=1}^\infty c_{kl}e^{i(ak+bl)t}=0$ for any $t\in \R$. This implies
    \begin{equation}
		0=||\sum_{k,l=1}^\infty c_{kl}e^{i(ak+bl)t}||_2^2=
		\inner{\sum_{k,l=1}^\infty c_{kl}e^{i(ak+bl)t}}{\sum_{k,l=1}^\infty c_{kl}e^{i(ak+bl)t}}=\sum_{k,l=1}^\infty |c_{kl}|^2.
\end{equation}
    The last step follows from the fact that $a$ and $b$ are linearly independent over $\Q$. Thus, $ak+bl$ are different for each $k,l\in\Z$ and, therefore, $(e^{i(ak+bl)x})_{k,l}$ is an orthonormal set. 
    This yields that $c_{kl}=0$ for all $k,l\in\N$ and, thus, $F=0$. This finishes the proof.
\end{proof}

\end{document}